\documentclass[10pt]{article}

\usepackage{ifpdf}
\ifpdf
\usepackage[pdftex]{hyperref}
\fi

\usepackage{amsmath}
\usepackage{amssymb}
\usepackage{epsf}
\usepackage{epsfig}
\usepackage{amsfonts}
\usepackage{amsthm}
\usepackage{enumerate} 

\newcommand{\Loss}{\mathop{{\rm{Loss}}}\nolimits}

\newcommand{\Kcal}{{\cal K}}

\newcommand{\Fcal}{{\cal F}}

\newcommand{\RR}{{\mathrm{RR}}}

\newcommand{\tr}{\mathop{{\rm{tr}}}\nolimits}

\newcommand{\R}{\mathbb R}
\newcommand{\N}{{\mathbb N}}

\newcommand{\calC}{{\cal C}}
\newcommand{\calN}{{\cal N}}

\newcommand{\cov}{\mathop{{\mathrm{cov}}}\nolimits}

\newtheorem{proposition}{Proposition}
\newtheorem{theorem}{Theorem}
\newtheorem{lemma}{Lemma}
\newtheorem{corollary}{Corollary}

\newtheorem{protocol}{Protocol}{\bfseries}{}

\theoremstyle{remark} 
\newtheorem{remark}{Remark}
\newcounter{append} 

\newcommand{\appendixsection}[1]{\refstepcounter{append}
{\section*{Appendix \Alph{append}. #1}}} 

\author{{\bf Fedor Zhdanov} and {\bf Yuri Kalnishkan} \\
Computer Learning Research Centre and \\
Department of   Computer Science, \\
  Royal Holloway, University of London, \\
Egham,  Surrey, TW20 0EX, United Kingdom \\
{\tt \{fedor,yura\}@cs.rhul.ac.uk}}

\title{An Identity for Kernel Ridge Regression\footnote{An earlier
    version of this paper appeared in Proceedings of ALT 2010,
    LNCS~6331, Springer, 2010. This paper also reproduces some results
    from technical report \cite{abs-0910-4683}.}}

\begin{document}

\maketitle

\begin{abstract}
This paper derives an identity connecting the square loss of ridge
regression in on-line mode with the loss of the retrospectively best
regressor. Some corollaries about the properties of the cumulative
loss of on-line ridge regression are also obtained.
\end{abstract}

\section{Introduction}

Ridge regression is a powerful technique of machine learning. It was
introduced in \cite{regression_hoerl}; the kernel version of it is
derived in \cite{vovk_ridge}.

Ridge regression can be used as a batch or on-line algorithm. This
paper proves an identity connecting the square losses of ridge
regression used on the same data in batch and on-line fashions. The
identity and the approach to the proof are not entirely new. The
identity implicitly appears in \cite{regression_azoury_warmuth} for
the linear case (it can be obtained by summing (4.21) from
\cite{regression_azoury_warmuth} in an exact rather than estimated
form). One of the proof methods based essentially on Bayesian
estimation features in \cite{regression_kakade2005}, which focuses on
probabilistic statements and stops one step short of formulating the
identity. In this paper we put it all together, explicitly formulate
the identity in terms of ridge regression, and give two proofs for the
kernel case.  The first proof is by calculating the likelihood in a
Gaussian processes model by different methods. It is remarkable that a
probabilistic proof yields a result that holds in the worst case along
any sequence of signals and outcomes with no probabilistic
assumptions. The other proof is based on the analysis of a
Bayesian-type algorithm for prediction with expert advice; it is
reproduced from unpublished technical report \cite{abs-0910-4683}.

We use the identity to derive several inequalities providing upper
bounds for the cumulative loss of ridge regression applied in the
on-line fashion. Corollaries \ref{corollary_kernel} and
\ref{corollary_linear} deal with the `clipped' ridge regression. The
later reproduces Theorem~4.6 from \cite{regression_azoury_warmuth}
(this result is often confused with Theorem~4 in \cite{vovk_cols},
which, in fact, provides a similar bound for an essentially different
algorithm). Corollary~\ref{cor:asympotic} shows that for continuous
kernels on compact domains the loss of (unclipped) on-line ridge
regression is asymptotically close to the loss of the retrospectively
best regressor.  This result cannot be generalised to non-compact
domains.

In the literature there is a range of specially designed
regression-type algorithms with better worst-case bounds or bounds
covering wider cases. Aggregating algorithm regression (also known as
Vovk-Azoury-Warmuth predictor) is described in \cite{vovk_cols},
\cite{regression_azoury_warmuth}, and Section~11.8 of
\cite{cbl_book_prediction}. Theorem~1 in \cite{vovk_cols} provides an
upper bound for aggregating algorithm regression; the bound is better
than the bound given by Corollary~\ref{corollary_linear} for clipped
ridge regression. The bound from \cite{vovk_cols} has also been shown
to be optimal in a strong sense. The exact relation between the
performance of ridge regression and the performance of aggregating
algorithm regression is not known. Theorem~3 in \cite{vovk_cols}
describes a case where aggregating algorithm regression performs
better, but in the case of unbounded signals. An important class of
regression-type algorithms achieving different bounds is based on the
gradient descent idea; see \cite{calif_wcql}, \cite{KivinenW97}, and
Section~11 in \cite{cbl_book_prediction}. In \cite{herbster_tracking}
and \cite{BusuttilK07} regression-type algorithms dealing with
changing dependencies are constructed. In \cite{chernov_discounted}
regression is considered within the framework of discounted loss,
which decays with time.

The paper is organised as follows. Section~\ref{sect_intro} introduces
kernels and kernel ridge regression in batch and on-line settings. We
use an explicit formula to introduce ridge regression;
Appendix~\ref{appendix_optimality} contains a proof that this formula
specifies the function with a certain optimality
property. Section~\ref{sect_identity} contains the statement of the
identity and Subsection~\ref{subsect_a0} shows that the
identity remains true (in a way) for the case of zero ridge.

Section~\ref{sect_corollaries} discusses corollaries of
the identity.
Section~\ref{sect_proof} contains the proof based on a probabilistic
interpretation of ridge regression in the context of Gaussian
fields. Section~\ref{sect_AA} contains an alternative proof based on
prediction with expert advice. The proof has been reproduced from
\cite{abs-0910-4683}. 

Appendixes~\ref{appendix_optimality}--\ref{appendix_cb_bound} to the
paper contain proofs of some known results; they have been included
for completeness and to clarify the intuition behind other proofs in
the paper.

\section{Kernel Ridge Regression in On-line and Batch Settings}
\label{sect_intro}

\subsection{Kernels}

A {\em kernel} on a domain $X$, which is an arbitrary set with no
structure assumed, is a symmetric positive-semidefinite function of
two arguments, i.e., $\Kcal: X\times X\to\R$ such that
\begin{enumerate}
\item for all $x_1,x_2\in X$ we have $\Kcal(x_1,x_2)=\Kcal(x_2,x_1)$
  and
\item for any positive integer $T$, any $x_1,x_2,\ldots,x_T\in X$ and
  any real numbers  $\alpha_1,\alpha_2,\ldots,\alpha_T\in\R$ we have
  $\sum_{i,j=1}^T\Kcal(x_i,x_j)\alpha_i\alpha_j\ge 0$.
\end{enumerate}
An equivalent definition can be given as follows. A function $\Kcal:
X\times X\to\R$ is a kernel if there is a Hilbert space $\Fcal$ of
functions on $X$ such that
\begin{enumerate}
\item for every $x\in X$ the function $\Kcal(x,\cdot)$, i.e., $\Kcal$
  considered as a function of the second argument with the first
  argument fixed, belongs to $\Fcal$ and
\item for every $x\in X$ and every $f\in\Fcal$ the value of $f$ at $x$
  equals the scalar product of $f$ by $\Kcal(x,\cdot)$, i.e.,
  $f(x)=\langle f,\Kcal(x,\cdot)\rangle_\Fcal$; this property is often
  called the {\em reproducing property}.
\end{enumerate}
The second definition is sometimes said to specify a {\em reproducing
  kernel}. The space $\Fcal$ is called the {\em reproducing kernel
  Hilbert space (RKHS)} for the kernel $\Kcal$ (it can be shown that
the RKHS for a kernel $\Kcal$ is unique). The equivalence of the two
definitions is proven in \cite{kern_aron43}.

\subsection{Regression in Batch and On-line Settings}
\label{subsect_RR}

Suppose that we are given a sample of pairs
$$S=((x_1,y_1),(x_2,y_2),\ldots,(x_T,y_T))\enspace,$$ where all
$x_t\in X$ are called {\em signals} and $y_t\in\R$ are called {\em
  outcomes} (or {\em labels}) for the corresponding signals. Every pair
$(x_t,y_t)$ is called {\em a (labelled) example}.

The task of regression is to fit a function (usually from a particular
class) to the data. The method of {\em kernel ridge regression} with a
kernel $\Kcal$ and a real regularisation parameter ({\em
  ridge}) $a>0$  suggests the function $f_\RR(x)={Y'(K+aI)^{-1}k(x)}$, where
$Y=(y_1,y_2,\ldots,y_T)'$ is the column vector\footnote{Throughout
  this paper $M'$ denotes the transpose of a matrix $M$.} of outcomes,
\begin{equation*}
K=
\begin{pmatrix}
\Kcal(x_1,x_1) & \Kcal(x_1,x_2) & \ldots & \Kcal(x_1,x_T)\\
\Kcal(x_2,x_1) & \Kcal(x_2,x_2) & \ldots & \Kcal(x_2,x_T)\\
\vdots         & \vdots        & \ddots & \vdots \\
\Kcal(x_T,x_1) & \Kcal(x_T,x_2) & \ldots & \Kcal(x_T,x_T)
\end{pmatrix}
\end{equation*}
is the {\em kernel matrix} and
\begin{equation*}
k(x)=
\begin{pmatrix}
\Kcal(x_1,x)\\
\Kcal(x_2,x)\\
\vdots      \\
\Kcal(x_T,x)
\end{pmatrix}\enspace.
\end{equation*}
Note that the matrix $K$ is positive-semidefinite by the definition of
a kernel, therefore the matrix $K+aI$ is positive-definite and thus
non-singular.

If the sample $S$ is empty, i.e., $T=0$ or no examples have been
given to us yet, we assume that $f_\RR(x)=0$ for all $x$.

It is easy to see that $f_\RR(x)$ is a linear combination of functions
$\Kcal(x_t,x)$ (note that $x$ does not appear outside of $k(x)$ in the
ridge regression formula) and therefore it belongs to the RKHS $\Fcal$
specified by the kernel $\Kcal$. It can be shown that on $f_\RR$ the
minimum of the expression $\sum_{t=1}^T(f(x_t)-y_t)^2+a\|f\|^2_\Fcal$
(where $\|\cdot\|_\Fcal$ is the norm in $\Fcal$) over all $f$ from the
RKHS $\Fcal$ is achieved. For completeness, we include a proof in
Appendix~\ref{appendix_optimality}.

Suppose now that the sample is given to us example by example. For
each example we are shown the signal and then asked to produce a
prediction for the outcome. One can say that the learner operates
according to the following protocol:

\begin{protocol}
\label{prot_online}
{\tt
\begin{tabbing}
  \quad\= for \=\quad\kill\\
  for $t=1,2,\ldots$\\
  \> read signal $x_t$\\
  \> output prediction $\gamma_t$\\
  \> read true outcome $y_t$\\
  endfor
\end{tabbing}
}
\end{protocol}

This learning scenario is called {\em on-line} or {\em
  sequential}. The scenario when the whole sample is given to us at
once as before is called {\em batch} to distinguish it from on-line.

One can apply ridge regression in the on-line scenario in the
following natural way. On step $t$ we form the sample $S_{t-1}$ from the
$t-1$ known examples $(x_1,y_1),(x_2,y_2),\ldots,(x_{t-1},y_{t-1})$
and output the prediction suggested by the ridge regression function
for this sample.

For the on-line scenario on step $t$ we will use the same notation as
in the batch mode but with the index $t-1$ denoting the
time\footnote{The conference version of this paper used $t$ rather
  than $t-1$. This paper uses $t-1$ because it coincides with the size
  and for compatibility with earlier papers.}. Thus $K_{t-1}$ is the
kernel matrix on step $t$ (its size is $(t-1)\times (t-1)$), $Y_{t-1}$
is the vector of outcomes $y_1,y_2,\ldots,y_{t-1}$, and
$k_{t-1}(x_t)=(\Kcal(x_1,x_t),\Kcal(x_2,x_t),\ldots,\Kcal(x_{t-1},x_t))'$
is $k(x_t)$ for step $t$. We will be referring to the prediction
output by on-line ridge regression on step $t$ as $\gamma_t^\RR$.

\section{The Identity}
\label{sect_identity}

\begin{theorem}
\label{th_main}
Take a kernel $\Kcal$ on a domain $X$ and a parameter $a>0$. Let
$\Fcal$ be the RKHS for the kernel $\Kcal$. For a sample
$(x_1,y_1),(x_2,y_2),\ldots,(x_T,y_T)$ let
$\gamma_1^\RR,\gamma_2^\RR,\ldots,\gamma_T^\RR$ be the predictions output by ridge
regression with the kernel $\Kcal$ and the parameter $a$ in the
on-line mode. Then
\begin{equation*}
\sum_{t=1}^T\frac{(\gamma^\RR_t-y_t)^2}{1+d_t/a}=\min_{f\in\Fcal}\left(
\sum_{t=1}^T(f(x_t)-y_t)^2+a\|f\|^2_\Fcal\right)
=aY_{T}'(K_{T}+aI)^{-1}Y_{T}\enspace,
\end{equation*}
where
$d_t=\Kcal(x_{t},x_t)-k'_{t-1}(x_t)(K_{t-1}+aI)^{-1}k_{t-1}(x_t)\ge 0$
and all other notation is as above.
\end{theorem}

The left-hand side term in this equality is {\em close} to the
cumulative square loss of ridge regression in the on-line mode. The
difference is in the denominators ${1+d_t/a}$. The values $d_t$ have
the meaning of variances of ridge regression predictions according to
the probabilistic view discussed below. Lemma~\ref{lemma_dt_to_0}
shows that $d_t\to 0$ as $t\to\infty$ for continuous kernels on
compact domains. The terms of the identity thus become close to the
cumulative square loss asymptotically; this intuition is formalised by
Corollary~\ref{corollary_kernel}.

Note that the minimum in the middle term is attained on $f$ specified
by batch ridge regression knowing the whole sample. It is thus {\em
  nearly} the square loss of the {\em retrospectively} best fit
$f\in\Fcal$.

The right-hand side term is a simple closed-form expression.

\subsection{The Case of Zero Ridge}
\label{subsect_a0}

In this subsection we show that the identity essentially remains true
for $a=0$.

Let the parameter $a$ in the identity approach 0. One may think that
the third term of the identity should tend to zero. On the other hand,
the value of the middle term of the identity for $a=0$ depends on $Y_T$;
the values of $y_t$ can be chosen (at least in some cases) so that
there is no exact fit in the RKHS (i.e., no $f\in\Fcal$ such that
$f(x_t)=y_t$, $t=1,2,\ldots,T$) and the middle term is not equal to
0. This section resolves the apparent contradiction.

As a matter of fact, the limit of the identity as $a\to 0$ does not
have to be 0. The situation when there is no exact fit in the RKHS is
only possible when the matrix $K_T$ is singular, and in this situation
the right-hand side does not always tend to 0.

The expression on the left-hand side of the identity is formally
undefined for $a=0$. The expression on the right-hand side is
undefined when $a=0$ and $K_T$ is singular.  The expression in the
centre, by contrast, always makes sense. The following theorem
clarifies the situation.

\begin{corollary}
Under the conditions of Theorem~\ref{th_main}, as $a\to 0$, the
terms of the identity 
\begin{equation*}
\sum_{t=1}^T\frac{(\gamma^\RR_t-y_t)^2}{1+d_t/a}=\min_{f\in\Fcal}\left(
\sum_{t=1}^T(f(x_t)-y_t)^2+a\|f\|^2_\Fcal\right)
=aY_{T}'(K_{T}+aI)^{-1}Y_{T}\enspace,
\end{equation*}
tend to the squared norm of the projection of
the vector $Y_T$ to the null space of the matrix $K_{T}$. This coincides
with the value of the middle term of the identity for $a=0$.
\end{corollary}

The {\em null space} (also called the {\em kernel}) of a matrix $S$ is
the space of vectors $x$ such that $Sx=0$. It is easy to see that the
dimension of the null space and the rank of $S$ (equal to the
dimension of the span of the columns of $S$) sum up to the number of
columns of $S$. If, moreover, $S$ is square and symmetric, the null
space of $S$ is the orthogonal complement of the span of the columns
of $S$.

\begin{proof}
For every $a\ge 0$ let $m_a=\inf_{f\in\Fcal}
\left(\sum_{t=1}^T(f(x_t)-y_t)^2+a\|f\|^2_\Fcal\right)$. Proposition~\ref{prop_KRR}
implies that if $a>0$ then the infimum is achieved on the ridge
regression function with the parameter $a$. Throughout this proof we
will denote this function by $f_a$.

Let us calculate the value of $m_0=\inf_{f\in\Fcal}
\sum_{t=1}^T(f(x_t)-y_t)^2$. It follows from the representer theorem
(see Proposition~\ref{prop_representer}) that it is sufficient to
consider the functions $f$ of the form
$f(\cdot)=\sum_{i=1}^Tc_i\Kcal(x_i,\cdot)$.

For $f(\cdot)=\sum_{t=1}^Tc_i\Kcal(x_t,\cdot)$ the sum
$\sum_{t=1}^T(f(x_t)-y_t)^2$ equals the squared norm
$\|K_{T}C-Y_T\|^2$, where $C=(c_1,c_2,\ldots,c_T)'$ is the vector of
coefficients of the linear combination. If $C_0$ minimises this
expression, then $K_{T}C_0$ is the projection of $Y_T$ to the linear
span of the columns of $K_{T}$. The vector $Y_T-K_{T}C_0$ is then
the projection of $Y_T$ to the orthogonal complement of the span and the
orthogonal complement is the null space of $K_{T}$.

Let us show that $m_a$ is continuous at $a=0$. Fix some $f_0$ such
that the infimum $m_0$ is achieved on $f_0$ (if $K_{T}$ is singular
there can be more than one such function). Substituting $f_0$ into the
formula for $m_a$ yields $m_a\le m_0+a\|f_0\|^2_\Fcal= m_0+o(1)$ as
$a\to 0$. Substituting $f_a$ into the definition of $m_0$ yields
$m_0\le m_a$. We thus get that $m_a\to m_0$ as $a\to 0$.
\end{proof}

\section{Corollaries}
\label{sect_corollaries}

In this section we use the identity to obtain some properties of
cumulative losses of on-line algorithms.

\subsection{A Multiplicative Bound}

It is easy to obtain a basic multiplicative bound on the loss of
on-line ridge regression. The matrix $(K_{t-1}+aI)^{-1}$ is
positive-definite as the inverse of a positive-definite, therefore
$k'_{t-1}(x_t)(K_{t-1}+aI)^{-1}k_{t-1}(x_t)\ge 0$ and
$d_t\le\Kcal(x_t,x_t)$. Assuming that there is $c_\Fcal>0$ such that
$\Kcal(x,x)\le c_\Fcal^2$ on $X$ (i.e., the evaluation functional on
$\Fcal$ is uniformly bounded by $c_\Fcal$), we get
\begin{multline}
\label{bound_mult}
\sum_{t=1}^T(\gamma^\RR_t-y_t)^2\le \left(1+\frac{c_\Fcal^2}{a}\right)
\min_{f\in\Fcal}\left(\sum_{t=1}^T(f(x_t)-y_t)^2+a\|f\|^2_\Fcal\right)=\\
a\left(1+\frac{c_\Fcal^2}{a}\right)Y'_{T}(K_{T}+aI)^{-1}Y_{T}\enspace.
\end{multline}

\subsection{Additive Bounds for Clipped Regression}

Some less trivial bounds can be obtained under the following
assumption.  Suppose that we know in advance that outcomes $y$ come
from an interval $[-Y,Y]$, and $Y$ is known to us. It does not make
sense then to make predictions outside of the interval. One may
consider {\em clipped ridge regression}, which operates as
follows. For a given signal the ridge regression prediction
$\gamma^\RR$ is calculated; if it falls inside the interval, it is
kept; if it is outside of the interval, it is replaced by the closest
point from the interval. Denote the prediction of clipped ridge
regression by $\gamma^{\RR,Y}$. If $y\in [-Y,Y]$ indeed holds, then
$(\gamma^{\RR,Y}-y)^2\le(\gamma^{\RR}-y)^2$ and
$(\gamma^{\RR,Y}-y)^2\le 4Y^2$.

\begin{corollary}
\label{corollary_kernel}
Take a kernel $\Kcal$ on a domain $X$ and a parameter $a>0$. Let
$\Fcal$ be the RKHS for the kernel $\Kcal$. For a sample
$(x_1,y_1),(x_2,y_2),\ldots,(x_T,y_T)$ such that $y_t\in [-Y,Y]$ for
all $t=1,2,\ldots,T$, let
$\gamma_1^{\RR,Y},\gamma_2^{\RR,Y},\ldots,\gamma_T^{\RR,Y}$ be the
predictions output by clipped ridge regression with the kernel
$\Kcal$ and the parameter $a$ in the on-line mode. Then
\begin{multline*}
\sum_{t=1}^T(\gamma^{\RR,Y}_t-y_t)^2\le  \\ \min_{f\in\Fcal}\left(
\sum_{t=1}^T(f(x_t)-y_t)^2+a\|f\|^2_\Fcal\right)+4Y^2\ln\det\left(I+\frac
1a K_{T}\right)\enspace,
\end{multline*}
where $K_{T}$ is as above.
\end{corollary}

\begin{proof}
We have
\begin{equation*}
 \frac 1{1+d_t/a}=1-\frac{d_t/a}{1+d_t/a}
\end{equation*}
and
\begin{equation*}
\frac{d_t/a}{1+d_t/a}\le\ln(1+d_t/a)\enspace;
\end{equation*}
indeed, for $b\ge 0$ the inequality $b/(1+b)\le\ln(1+b)$ holds and can
be checked by differentiation. Therefore
\begin{align*}
\sum_{t=1}^T(\gamma^{\RR,Y}_t-y_t)^2 &=
\sum_{t=1}^T(\gamma^{\RR,Y}_t-y_t)^2\frac 1{1+d_t/a}+
\sum_{t=1}^T(\gamma^{\RR,Y}_t-y_t)^2\frac{d_t/a}{1+d_t/a}\\
&\le\sum_{t=1}^T(\gamma^{\RR}_t-y_t)^2\frac 1{1+d_t/a}+
4Y^2\sum_{t=1}^T\ln(1+d_t/a)\enspace.
\end{align*}
Lemma \ref{lem_proddt} proved below yields
\begin{equation*}
\prod_{t=1}^T(1+d_t/a)=\frac 1{a^T}\det(K_{T}+aI)=\det\left(I+\frac 1a
K_{T}\right) \enspace.
\end{equation*}
\end{proof}

There is no sub-linear upper bound on the regret term
$$4Y^2\ln\det\left(I+\frac 1a K_{T}\right)$$
in the general case; indeed,
consider the kernel
\begin{equation}
\delta(x_1,x_2)=
\label{eq_delta}
\begin{cases}
1 & \text{if $x_1=x_2$;}\\
0 & \text{otherwise.}
\end{cases}
\end{equation}
However we can  get good bounds in special  cases.

It is shown in \cite{kakade_seeger_long} that for the Gaussian kernel
$\Kcal(x_1,x_2)=e^{-b\|x_1-x_2\|^2}$, where $x_1,x_2\in\R^d$, we can
get an upper bound on average. Suppose that all $x$s are independently
identically distributed according to the Gaussian distribution with
the mean of 0 and variance of $cI$. Then for the expectation we have
$E\ln\det\left(I+\frac 1a K_{T}\right)=O((\ln T)^{d+1})$ (see
Section~IV.B in \cite{kakade_seeger_long}). This yields a bound on the
expected loss of clipped ridge regression.

Consider the linear kernel $\Kcal(x_1,x_2)=x'_1x_2$ defined on column
vectors from $\R^n$. We have $\Kcal(x,x)=\|x\|^2$, where $\|\cdot\|$
is the quadratic norm in $\R^n$. The reproducing kernel Hilbert space
is the set of all linear functions on $\R^n$. We have $K_t=X_t'X_t$,
where $X_{T}$ is the {\em design matrix} made up of column vectors
$x_1,x_2,\ldots,x_{T}$. The Sylvester determinant identity
$\det(I+UV)=\det(I+VU)$ (see, e.g., \cite{matan_henderson}, Eq.~(6))
implies that
\begin{equation*}
\det\left(I+\frac 1a X'_{T}X_{T}\right)=\det\left(I+
\frac 1a X_{T}X_{T}'\right)=
\det\left(I+\frac 1a\sum_{t=1}^{T}x_tx'_t\right)\enspace.
\end{equation*}
We will use an upper bound from \cite{cesa-bianchi_perceptron} for
this determinant; a proof is given in Appendix~\ref{appendix_cb_bound}
for completeness.  We get the following corollary.
\begin{corollary}
\label{corollary_linear}
For a sample $(x_1,y_1),(x_2,y_2),\ldots,(x_T,y_T)$, where $\|x_t\|\le
B$ and $y_t\in [-Y,Y]$ for all $t=1,2,\ldots,T$, let
$\gamma_1^{\RR,Y},\gamma_2^{\RR,Y},\ldots,\gamma_T^{\RR,Y}$ be the
predictions output by clipped linear ridge regression with a parameter
$a>0$ in the on-line mode. Then
\begin{equation*}
\sum_{t=1}^T(\gamma^{\RR,Y}_t-y_t)^2\le \min_{\theta\in\R^n}\left(
\sum_{t=1}^T(\theta'x_t-y_t)^2+a\|\theta\|^2\right)+
4Y^2n\ln\left(1+\frac{TB^2}{an}\right)\enspace.
\end{equation*}
\end{corollary}

It is an interesting problem if the bound is optimal. As far as we
know, there is a gap in existing bounds. Theorem~2 in \cite{vovk_cols}
shows that $Y^2n\ln T$ is a lower bound for {\em any} learner and in
the constructed example, $\|x_t\|_\infty=1$. Theorem~3 in
\cite{vovk_cols} provides a stronger lower bound, but at the cost of
allowing unbounded $x$s.

\subsection{An Asymptotic Comparison}

The inequalities we have considered so far hold for finite time
horizons $T$. We shall now let $T$ tend to infinity.

Let us analyse the behaviour of the quantity
\begin{equation*}
d_t=\Kcal(x_{t},x_t)-k'_{t-1}(x_t)(K_{t-1}+aI)^{-1}k_{t-1}(x_t)\enspace.
\end{equation*}
According to the probabilistic interpretation discussed in
Subsection~\ref{subsect_Gauss}, $d_t$ has the meaning of the variance
of the prediction output by kernel ridge regression on step $t$ and
therefore it is always non-negative.

The probabilistic interpretation suggests that the variance should go
down with time as we learn the data better. In the general case this
is not true. Indeed, if $\Kcal=\delta$ defined by (\ref{eq_delta}) and
all $x_i$ are different then $d_t=1$ for all $t=1,2,\ldots$ However
under natural assumptions that hold in most reasonable cases the
following lemma holds. The lemma generalises Lemma~A.1 from
\cite{regression_kumon} because for the linear kernel $d_t$ can be can
be represented as shown in (\ref{dt_linear}) below.

\begin{lemma}
\label{lemma_dt_to_0}
Let $X$ be a compact metric space and a kernel $\Kcal: X^2\to\R$ be
continuous in both arguments. Then for any sequence $x_1,x_2,\ldots\in
X$ and $a>0$ we have $d_t\to 0$ as $t\to\infty$.
\end{lemma}

\begin{proof}
As discussed in Subsection~\ref{subsect_Gauss}, $d_t$ has the meaning
of a dispersion under a certain probabilistic interpretation and
therefore $d_t\ge 0$. One can easily see that
$k'_{t-1}(x_t)(K_{t-1}+aI)^{-1}k_{t-1}(x_t)\ge 0$. Indeed, the matrix
$(K_{t-1}+aI)^{-1}$ is positive-definite as the inverse of a
positive-definite. We get
\begin{equation}
\label{eq_dt_pos}
0\le k'_{t-1}(x_t)(K_{t-1}+aI)^{-1}k_{t-1}(x_t)\le \Kcal(x_t,x_t)\enspace.
\end{equation}

Let us start by considering a special case. Suppose that the sequence
$x_1,x_2,\ldots$ converges. Let $\lim_{t\to\infty}x_t=x_0\in X$. The
continuity of $\Kcal$ implies that $\Kcal(x_t,x_t)\to \Kcal(x_0,x_0)$
and
\begin{equation}
\label{eq_dt_pos_o}
0\le k'_{t-1}(x_t)(K_{t-1}+aI)^{-1}k_{t-1}(x_t)\le \Kcal(x_0,x_0)+o(1)
\end{equation}
as $t\to\infty$.  The positive-semidefiniteness clause in the
definition of a kernel implies that $\Kcal(x_0,x_0)\ge 0$. We will
obtain a lower bound on $k'_{t-1}(x_t)(K_{t-1}+aI)^{-1}k_{t-1}(x_t)$
and show that it must converge to $\Kcal(x_0,x_0)$ thus proving the
lemma in the special case.

For every symmetric positive-semidefinite matrix $M$ and a vector $x$
of the matching size we have $\lambda_\mathrm{min}\|x\|^2\le x'Mx$, where
$\lambda_\mathrm{min}\ge 0$ is the smallest eigenvalue of $M$ (this can be
shown by considering the orthonormal base where $M$
diagonalises). The smallest eigenvalue of $(K_{t-1}+aI)^{-1}$ equals
$1/(\tilde\lambda+a)$, where $\tilde\lambda$ is the largest eigenvalue
of $K_{t-1}$. The value of $\tilde\lambda$ is bounded from above by
the trace of $K_{t-1}$:
\begin{equation*}
\tilde\lambda\le \sum_{i=1}^{t-1}\Kcal(x_i,x_i)
\end{equation*}
and this yields a lower bound on the smallest eigenvalue of
$(K_{t-1}+aI)^{-1}$.

The squared norm of $k_{t-1}$ equals
\begin{equation*}
\|k_{t-1}(x_t)\|^2 = \sum_{i=1}^{t-1}(\Kcal(x_i,x_t))^2\enspace.
\end{equation*}
Combining this with the above estimates we get
\begin{equation*}
k'_{t-1}(x_t)(K_{t-1}+aI)^{-1}k_{t-1}(x_t)\ge 
\frac{\sum_{i=1}^{t-1}(\Kcal(x_i,x_t))^2}{a+\sum_{i=1}^{t-1}\Kcal(x_i,x_i)}\enspace.
\end{equation*}
Let us assume $\Kcal(x_0,x_0)\ne 0$ and show that the right-hand size
of the inequality tends to $\Kcal(x_0,x_0)$ (if $\Kcal(x_0,x_0)=0$,
then (\ref{eq_dt_pos_o}) implies that
$k'_{t-1}(x_t)(K_{t-1}+aI)^{-1}k_{t-1}(x_t)\to 0$). Dividing the
numerator and the denominator by $t-1$ yields
\begin{equation*}
\frac{\sum_{i=1}^{t-1}(\Kcal(x_i,x_t))^2}{a+\sum_{i=1}^{t-1}\Kcal(x_i,x_i)}=
\frac{\frac{\sum_{i=1}^{t-1}(\Kcal(x_i,x_t))^2}{t-1}}
{\frac{a}{t-1}+\frac{\sum_{i=1}^{t-1}\Kcal(x_i,x_i)}{t-1}}\enspace.
\end{equation*}
Clearly, as $t$ goes to infinity, most terms in the sums become
arbitrary close to $(\Kcal(x_0,x_0))^2$ and $\Kcal(x_0,x_0)$ and thus
the `averages' tend to $(\Kcal(x_0,x_0))^2$ and $\Kcal(x_0,x_0)$,
respectively. Therefore the fraction tends to $\Kcal(x_0,x_0)$ and 
(\ref{eq_dt_pos_o}) implies that
$k'_{t-1}(x_t)(K_{t-1}+aI)^{-1}k_{t-1}(x_t)\to \Kcal(x_0,x_0)$. We
have shown that $d_t\to 0$ for the special case when the sequence
$x_1,x_2,\ldots$ converges.

Let us prove the lemma for an arbitrary sequence $x_1,x_2,\ldots\in X$. 
If $d_t\not\to 0$, there is a subsequence of indexes
$\tau_1<\tau_2<\ldots<\tau_k<\ldots$ such that $d_{\tau_k}$ is separated from
0. Since $X$ is compact, we can choose a sub-subsequence
$t_1<t_2<\ldots<t_n<\ldots$ such that $d_{t_n}$ is separated from 0
and $x_{t_n}$ converges. If we show that $d_{t_n}\to 0$, we get a
contradiction and prove the lemma. Thus it is sufficient to show that
$d_{t_n}\to 0$ where $\lim_{n\to\infty}x_{t_n}=x_0\in X$.

Clearly, we have inequalities (\ref{eq_dt_pos_o}) for $t=t_n$:
\begin{equation}
\label{eq_dt_sub}
0\le k'_{t_n-1}(x_{t_n})(K_{t_n-1}+aI)^{-1}k_{t_n-1}(x_{t_n})\le
\Kcal(x_0,x_0)+o(1)\enspace. 
\end{equation}
We proceed by obtaining a lower bound on the middle term as before.

Fix some $n$ and the corresponding $t_n$. One can rearrange the order of
the elements of the finite sequence $x_1,x_2,\ldots,x_{t_n-1}$ to put
the elements of the subsequence to the front and consider the
sequence (still of length $t_n-1$) $x_{t_1},x_{t_2},\ldots,x_{t_{n-1}},\bar x_1,\bar
x_2,\ldots,\bar x_{t_n-n}$, where $\bar x_1,\bar x_2,\ldots,\bar
x_{t_n-n}$ are the elements of the original sequence with indexes not
in the set $\{t_1,t_2,\ldots,t_{n-1}\}$.

One can write
\begin{equation*}
k'_{t_n-1}(x_{t_n})(K_{t_n-1}+aI)^{-1}k_{t_n-1}(x_{t_n})=
\tilde k'_{t_n-1}(x_{t_n})(\widetilde{K}_{t_n-1}+aI)^{-1}\tilde
k_{t_n-1}(x_{t_n}) \enspace,
\end{equation*}
where
\begin{equation*}
\tilde k'_{t_n-1}(x_{t_n}) =
\begin{pmatrix}
\Kcal(x_{t_1},x_{t_n})\\
\vdots\\
\Kcal(x_{t_{n-1}},x_{t_n})\\
\Kcal(\bar x_{1},x_{t_n})\\
\vdots\\
\Kcal(\bar x_{t_n-n},x_{t_n})\\
\end{pmatrix}
\end{equation*}
and
\begin{multline*}
\widetilde{K}_{t_n-1}=\\
\begin{pmatrix}
\Kcal(x_{t_1},x_{t_1}) & \ldots & \Kcal(x_{t_1},x_{t_{n-1}}) &
\Kcal(x_{t_1},\bar x_{1}) & \ldots & \Kcal(x_{t_1},\bar x_{t_n-n})\\
\vdots & & \vdots & \vdots & & \vdots\\
\Kcal(x_{t_{n-1}},x_{t_1}) & \ldots & \Kcal(x_{t_{n-1}},x_{t_{n-1}}) &
\Kcal(x_{t_{n-1}},\bar x_{1}) & \ldots & \Kcal(x_{t_{n-1}},\bar x_{t_n-n})\\
\Kcal(\bar x_1,x_{t_1}) & \ldots & \Kcal(\bar x_1,x_{t_{n-1}}) &
\Kcal(\bar x_1,\bar x_{1}) & \ldots & \Kcal(\bar x_1,\bar x_{t_n-n})\\
\vdots & & \vdots & \vdots & & \vdots\\
\Kcal(\bar x_{t_n-n},x_{t_1}) & \ldots & \Kcal(\bar x_{t_n-n},x_{t_{n-1}}) &
\Kcal(\bar x_{t_n-n},\bar x_{1}) & \ldots & \Kcal(\bar x_{t_n-n},\bar x_{t_n-n})\\
\end{pmatrix}.
\end{multline*}
Indeed, we can consider the matrix product
$k_{t_n-1}(x_{t_n})(K_{t_n-1}+aI)^{-1}k_{t_n-1}(x_{t_n})$ in the
rearranged orthonormal base where the base vectors with the indexes
$t_1,t_2,\ldots,t_{n-1}$ are at the front of the list. (Alternatively
one can check that rearranging the elements of the training set does
not affect ridge regression prediction and its variance.)

The upper left corner of $\widetilde{K}_{t_n-1}$ and the upper part of
$\tilde k_{t_n-1}(x_{t_n})$ consist of values of the kernel on
elements of the subsequence, $\Kcal(x_{t_i},x_{t_j}),
i,j=1,2,\ldots,n$. We shall use this observation and reduce the proof
to the special case considered above.

Let us single out the left upper corner of size $(n-1)\times (n-1)$ in
$\widetilde{K}_{t_n-1}+aI$ and apply Lemma~\ref{lemma_partition} from
Appendix~\ref{appendix_cb_bound}.  The special case considered above
implies that $\tilde
k'_{t_n-1}(x_{t_n})(\widetilde{K}_{t_n-1}+aI)^{-1}\tilde
k_{t_n-1}(x_{t_n})\ge\Kcal(x_0,x_0)+o(1)$ as $n\to\infty$. Combined
with (\ref{eq_dt_sub}) this proves that $d_{t_n}\to 0$ as $n\to 0$.

\end{proof}

We shall apply this lemma to establish asymptotic equivalence between
the losses in on-line and batch cases. The following corollary
generalises Corollary~3 from \cite{abs-0910-4683}.

\begin{corollary}\label{cor:asympotic}
Let $X$ be a compact metric space and a kernel $\Kcal: X^2\to\R$ be
continuous in both arguments; let $\Fcal$ be the RKHS corresponding to
the kernel $\Kcal$. For a sequence $(x_1,y_1), (x_2,y_2),\ldots\in
X\times\R$ let $\gamma_t^\RR$ be the predictions output by on-line
ridge regression with a parameter $a>0$. Then
\begin{enumerate}
\item if there is $f\in\Fcal$ such that $\sum_{t=1}^{\infty}(y_t -
  f(x_t))^2< +\infty$ then
$$
\sum_{t=1}^{\infty} (y_t - \gamma_t^\RR)^2 < +\infty\enspace;
$$
\item if for all $f\in\Fcal$ we have $\sum_{t=1}^{\infty} (y_t -
  f(x_t))^2=+\infty$, then
\begin{equation}\label{eq:conclusion}
\lim_{T\to\infty}
\frac{\sum_{t=1}^T(y_t - \gamma_t^\RR)^2}{\min_{f\in\Fcal}
\left(\sum_{t=1}^T(y_t - f(x_t))^2 + a\|f\|_\Fcal^2\right)}=1\enspace.
\end{equation}
\end{enumerate}
\end{corollary}

\begin{proof}
Part~1 follows from bound~(\ref{bound_mult}).  Indeed, the continuous
function $\Kcal(x,x)$ is uniformly bounded on $X$ and one can take a
finite constant $c_\Fcal$.

Let us prove Part~2.  We observed above that $d_t\ge 0$. The identity
implies
\begin{equation*}
\sum_{t=1}^T(y_t - \gamma_t^\RR)^2\ge \sum_{t=1}^T\frac{(y_t -
  \gamma_t^\RR)^2}{1+d_t/a}=\min_{f\in\Fcal}
\left(\sum_{t=1}^T(y_t - f(x_t))^2 + a\|f\|_\Fcal^2\right)
\end{equation*}
and thus the fraction in (\ref{eq:conclusion}) is always greater than
or equal to 1.

Let $m_t=\min_{f\in\Fcal} \left(\sum_{\tau=1}^t(y_\tau - f(x_\tau))^2
+ a\|f\|_\Fcal^2\right)$. The sequence $m_t$ is non-decreasing. Indeed
let the minimum in the definition of $m_t$ be achieved on $f_t$. If
$m_{t+1}<m_t$ then one can substitute $f_{t+1}$ into the definition of
$m_t$ and decrease the minimum. 


Let us prove that $m_t\to+\infty$ as $t\to\infty$.  A monotonic
sequence must have a limit; let $\lim_{t\to\infty}m_t=m_\infty$. We
have $m_1\le m_2\le\ldots\le m_\infty$.  We will assume that
$m_\infty<+\infty$ and show that there is $f_\infty\in\Fcal$ such that
$\sum_{t=1}^{\infty} (y_t - f_\infty(x_t))^2\le m_\infty<+\infty$
contrary to the condition of Part~2.

Proposition~\ref{prop_KRR} implies that $f_t(\cdot)$ is the ridge
regression function and belongs to the linear span of
$\Kcal(x_i,\cdot)$, $i=1,2,\ldots,t$, which we will denote by $X_t$.
(For uniformity let $X_0=\varnothing$ and $m_0=0$.)  The squared norm
of $f_t$ does not exceed $m_t/a\le m_\infty/a$. Thus all $f_t$ belong
to the ball of radius $\sqrt{m_\infty/a}$ centred at the origin.

Let $X_\infty$ be the closure of the linear span of
$\bigcup_{t=0}^{\infty}X_t$.  If $X_\infty$ happens to have a finite
dimension, then all $f_t$ belong to a ball in a finite-dimensional
space; this ball is a compact set. If $X_\infty$ is of infinite dimension,
the ball is not compact but we will construct a different compact set
containing all $f_t$.

Take $0\le s<t$. The function $f_t$ can be uniquely decomposed as
$f_t=g+h$, where $g$ belongs to $X_s$ and $h$ is orthogonal to $X_s$.
The Pythagoras theorem implies that
$\|f\|_\Fcal^2=\|g\|_\Fcal^2+\|h\|_\Fcal^2$.  The function $h(\cdot)$
is orthogonal to all $\Kcal(x_i,\cdot)$, $i=1,2,\ldots,s$; thus
$h(x_i)=\langle h,\Kcal(x_i,\cdot)\rangle=0$ and $f_t(x_i)=g(x_i)$,
$i=1,2,\ldots,s$ (recall the proof of the representer theorem,
Proposition~\ref{prop_representer}). Note that $g$ cannot outperform
$f_s$, which achieves the minimum $m_s$. We get
\begin{align*}
m_t &= \sum_{i=1}^t(y_i-f_t(x_i))^2+a\|f_t\|^2_\Fcal\\
&= \sum_{i=1}^s(y_i-g(x_i))^2 + a\|g\|^2_\Fcal +
\sum_{i=s+1}^t(y_i-f_t(x_i))^2+a\|h\|^2_\Fcal\\
&\ge m_s +\sum_{i=s+1}^t(y_i-f_t(x_i))^2+a\|h\|^2_\Fcal\enspace.
\end{align*}
This inequality implies that $\|h\|^2_\Fcal \le (m_t-m_s)/a\le (m_\infty-m_s)/a$.

Consider the set $B$ of functions $f\in X_\infty\subseteq\Fcal$ satisfying
the following property for every $s=0,1,2,\ldots$: let $f=g+h$ be the
unique decomposition such that $g\in X_s$ and $h$ is orthogonal to
$X_s$; then the norm of $h$ satisfies $\|h\|_\Fcal^2\le (m_\infty-m_s)/a$.

We have shown that all $f_t$ belong to $B$, $t=1,2,\ldots$ Let us show
that $B$ is compact. It is closed because the projection in Hilbert
spaces is a continuous operator. Let us show that $B$ is totally
bounded. We shall fix $\varepsilon>0$ and construct a finite
$\varepsilon$-net of points in $B$ such that $B$ is covered by closed
balls of radius $\varepsilon$ centred at the points from the net.

There is $s>0$ such that $(m_\infty-m_s)/a\le \varepsilon^2/2$ because
$m_s\to m_\infty$. The ball of radius $\sqrt{(m_\infty-m_0)/a}$ in
$X_s$ is compact and therefore it contains a finite
$\varepsilon/\sqrt{2}$-net $g_1,g_2,\ldots,g_k$.  Every $f\in B$ can
be represented as $f=g+h$, where $g$ belongs to $X_s$ and $h$ is
orthogonal to $X_s$. Since $\|g\|^2_\Fcal\le \|f\|^2_\Fcal \le
(m_\infty-m_0)/a$, the function $g$ belongs to the ball of radius
$\sqrt{(m_\infty-m_0)/a}$ in $X_s$ and therefore $\|g-g_i\|_\Fcal\le
\varepsilon/\sqrt{2}$ for some $g_i$ from the
$\varepsilon/\sqrt{2}$-net. The definition of $B$ implies that
$\|h\|^2_\Fcal\le (m_\infty-m_s)/a\le\varepsilon^2/2$. The Pythagoras
theorem yields
\begin{align*}
\|f-g_i\|_\Fcal^2&=\|g-g_i\|_\Fcal^2+\|h\|^2_\Fcal\\
&\le \varepsilon^2/2+\varepsilon^2/2\\
&=\varepsilon^2\enspace.
\end{align*}
Thus the net we have constructed is an $\varepsilon$-net for $B$.

Since the functions $f_t$ belong to a compact set, there is a
converging sub-sequence $f_{t_k}$; let
$\lim_{k\to\infty}f_{t_k}=f_\infty$. We have $\sum_{t=1}^\infty(y_t -
f_\infty(x_t))^2 + a\|f_\infty\|_\Fcal^2\le m_\infty$. Indeed, if
$\sum_{t=1}^\infty(y_t - f_\infty(x_t))^2 + a\|f_\infty\|_\Fcal^2>m_\infty$ then for a
sufficiently large $T_0$ we have $\sum_{t=1}^{T_0}(y_t - f_\infty(x_t))^2 +
a\|f_\infty\|_\Fcal^2>m_\infty$. Since $f_{t_k}\to f_\infty$ we get $f_{t_k}(x)\to
f_\infty(x)$ for all $x\in X$ and for sufficiently large $k$ all
$f_{t_k}(x_t)$ are sufficiently close to $f_\infty(x_t)$,
$t=1,2,\ldots,T_0$ and $\|f_{t_k}\|_\Fcal$ is sufficiently close to
$\|f_\infty\|_\Fcal$ so that $\sum_{t=1}^{T_0}(y_t - f_{t_k}(x_t))^2 +
a\|f_{t_k}\|_\Fcal^2>m_\infty$. We have proved that under the conditions of
Part~2 we have $m_t\to+\infty$ as $t\to \infty$.

Take $\varepsilon>0$. Since by Lemma~\ref{lemma_dt_to_0} we have $d_T\to 0$,
there is $T_0$ such that for all $T\ge T_0$ we have $1+d_T/a\le
1+\varepsilon$ and
\begin{align*}
\sum_{t=1}^T (y_t-\gamma_t^\RR)^2 &= \sum_{t=1}^{T_0} (y_t-\gamma_t^\RR)^2
+ \sum_{t=T_0+1}^T (y_t-\gamma_t^\RR)^2 \\
& \le \sum_{t=1}^{T_0}(y_t-\gamma_t^\RR)^2 + (1+\epsilon) \sum_{t=1}^T
\frac{(y_t-\gamma_t^\RR)^2}{1+d_t/a}\\
& = \sum_{t=1}^{T_0}(y_t-\gamma_t^\RR)^2 + (1+\epsilon)\min_{f\in\Fcal} \left(
\sum_{t=1}^T (y_t-f(x_t))^2 + a\left\|f\right\|_\Fcal^2 \right)\enspace.
\end{align*}
Therefore for all sufficiently large $T$
the fraction in (\ref{eq:conclusion}) does not exceed $1+2\varepsilon$.
\end{proof}

\begin{remark}
The proof of compactness above is based on the following general
result (cf.~\cite{brown_page_functional_analysis}, Chapter~4,
exercise~7 on p.~172). Let $B$ be a subset of $l_2$. Then $B$ is
totally bounded if and only if there is a sequence of nonnegative
numbers $\alpha_1,\alpha_2,\ldots \ge 0$ converging to 0, i.e.,
$\lim_{t\to\infty}\alpha_t=0$, such that for every $x =
(x_1,x_2,\ldots)\in B$ and every $t=1,2,\ldots$ the inequality
$\sum_{i=t}^{\infty}x_i^2\le\alpha(t)$ holds. This result generalises
the well-known construction of the Hilbert cube (also known as the
Hilbert brick).
\end{remark}

The corollary does not hold for a non-compact domain. Let us construct
a counterexample.

Let $X$ be the unit ball in $l_2$, i.e., $X=\{x\in l_2\mid
\|x\|_{l_2}=1\}$.
Let the kernel on $X$ be the scalar product in $l_2$, i.e., for
$u=(u_1,u_2,\ldots)$ and $v=(v_1,v_2,\ldots)$ from $X$ we have
$\Kcal(u,v)=\langle u,v\rangle_{l_2}=\sum_{i=1}^{\infty}u_iv_i$.

Consider the following sequence of elements $x_t\in X$. Let
$x_{2i-1}=x_{2i}$ have one at position $i$ and zeroes elsewhere,
$i=1,2,\ldots$ Consider the sequence of outcomes where odd elements
equal 1 and even elements equal 0, i.e., $y_{2i-1}=1$ and $y_{2i}=0$
for $i=1,2,\ldots$ We get \\
\begin{tabular}{lcl}
$t$ & $x_t$ & $y_t$\\
$1$ & $(1,0,0,\ldots)$ & 1\\
$2$ & $(1,0,0,\ldots)$ & 0\\
$3$ & $(0,1,0,\ldots)$ & 1\\
$4$ & $(0,1,0,\ldots)$ & 0\\
    & \vdots           &
\end{tabular}

Fix $a>0$. Let us work out the predictions $\gamma_t^\RR$ output by
on-line ridge regression on this sequence. The definition implies that
$\gamma_1^\RR=0$ and $\gamma_2^\RR=1/(1+a)$. To obtain further
predictions we need the following lemma stating that examples with
signals orthogonal to {\em all other} signals and $x_0$ where we want
to obtain a prediction can be dropped from the sample.

\begin{lemma}
Let $\Kcal: X\times X\to\R$ be a kernel on a domain $X$; let
$S=((x_1,y_1),(x_2,y_2),\ldots,(x_T,y_T))\in (X\times\R)^*$ be a
sample of pairs and let $x_0\in X$. If there is a subset
$(x_{i_1},y_{i_1}),(x_{i_2},y_{i_2}),\ldots,(x_{i_k},y_{i_k})$ of $S$
such that the signals of the examples from this subset are orthogonal
w.r.t.\ $\Kcal$ to all other signals, i.e., $\Kcal(x_{i_j},x_m)=0$ for
all $j=1,2,\ldots,k$ and $m\ne i_1,i_2,\ldots,i_k$, and orthogonal to
$x_0$ w.r.t.\ $\Kcal$, i.e., $\Kcal(x_{i_j},x_0)=0$ for
all $j=1,2,\ldots,k$, then all elements of this subset can be removed
from the sample $S$ without affecting the ridge regression prediction
$f_\RR(x_0)$. 
\end{lemma}

\begin{proof} Let the subset coincide with the whole of $S$. Then
$k(x_0)=0$ and the ridge regression formula implies that ridge
regression outputs $\gamma=0$. Dropping the whole sample $S$ leads
to the same prediction by definition. For the rest of the proof
assume that the subset is proper.

The main part of the proof relies on the optimality of ridge
regression given by Proposition~\ref{prop_KRR}.  Let $\Fcal$ be the
RKHS of functions on $X$ corresponding to $\Kcal$.  The ridge
regression function for the sample $S$ minimises
$\sum_{t=1}^T(f(x_t)-y_t)^2+a\|f\|^2_\Fcal$ and by the representer
theorem (Proposition~\ref{prop_representer}) it is a linear
combination of $\Kcal(x_i,\cdot)$, $i=1,2,\ldots,T$.

Let us represent a linear combination $f$ as $f_1+f_2$, where $f_1$ is
a linear combination of $\Kcal(x_{i_j},\cdot)$, $j=1,2,\ldots,k$
corresponding to signals from the subset and $f_2$ is a linear
combination of the remaining signals. The functions $f_1$ and $f_2$
are orthogonal in $\Fcal$ and this representation is unique. For every
$j=1,2,\ldots,k$ and $m\ne i_1,i_2,\ldots,i_k$ we have $f(x_{i_j}) =
f_1(x_{i_j})$ and $f(x_{m}) = f_2(x_{m})$ and therefore
\begin{multline*}
\sum_{t=1}^T(f(x_t)-y_t)^2+a\|f\|^2_\Fcal =\\
\sum_{j=1}^k(f_1(x_{i_j})-y_{i_j})^2+
\sum_{m\ne i_1,i_2,\ldots,i_k}(f_2(x_m)-y_m)^2+a\|f_1\|^2_\Fcal+a\|f_2\|^2_\Fcal
\enspace.
\end{multline*}

This expression splits into two terms depending only on $f_1$ and
$f_2$. We can minimise it independently over $f_1$ and $f_2$. Note
that $f_1(x_0)=0$ by assumption and therefore $f_\RR(x_0)=\tilde f_2(x_0)$,
where $\tilde f_2$ minimises
\begin{equation*}
\sum_{m\ne i_1,i_2,\ldots,i_k}(f(x_m)-y_m)^2+a\|f\|^2_\Fcal
\end{equation*}
over $\Fcal$. The optimality property implies that $\tilde f_2$ is the
ridge regression function for the smaller sample.
\end{proof}

The lemma implies that $\gamma_{2i-1}=\gamma_1=0$ and
$\gamma_{2i}=\gamma_2=1/(1+a)$ for all $i=1,2,\ldots$ It is easy to
see that Corollary~\ref{cor:asympotic} is violated. For $f_0=0$ we
have
\begin{equation*}
\frac{\sum_{t=1}^{2i}(\gamma_t^\RR-y_i)^2}{\sum_{t=1}^{2i}(f_0(x_t)-y_t)^2}=
\frac{i(1+1/(1+a)^2)}{i}=1+\frac{1}{(1+a)^2}>1\enspace.
\end{equation*}
The actual minimiser\footnote{It can easily be calculated but we do
  not really need it.} gives an even smaller denominator and an even
larger fraction.

We have shown that compactness is necessary in
Corollary~\ref{cor:asympotic}. It is easy to modify the counterexample
to show that compactness without the continuity of $\Kcal$ is not
sufficient. Indeed, take an arbitrary compact metric space $X$
containing an infinite sequence $\tilde x_0,\tilde x_1,\tilde
x_2,\ldots$ where $\tilde x_i\ne\tilde x_j$ for $i\ne j$.  Let $\Phi:
X\to l_2$ be such that $\tilde x_i$ is mapped to $x_i$ from the
counterexample for every $i=1,2,\ldots$ and $x_0$ is mapped to
0. Define the kernel $\Kcal$ on $X^2$ by $\Kcal(u,v)=\langle
\Phi(u),\Phi(v)\rangle_{l_2}$ (this kernel cannot be continuous).
Take $y_i$ as in the counterexample. All predictions and losses on
$(\tilde x_1,y_1),(\tilde x_2,y_2),\ldots$ will be as in the
counterexample with $\Kcal(x_0,\cdot)$ playing the part of $f_0$.

\section{Gaussian Fields and a Proof of the Identity}
\label{sect_proof}

We will prove the identity by means of a probabilistic interpretation
of ridge regression.

\subsection{Probabilistic Interpretation}
\label{subsect_Gauss}

Suppose that we have a Gaussian random field\footnote{We use the term
  `field' rather than `process' to emphasise the fact that $X$ is not
  necessarily a subset of $\R$ and its elements do not have to be
  moments of time; some textbooks still use the word `process' in this
  case.} $z_x$ with the means of 0 and the covariances
$\cov(z_{x_1},z_{x_2})=\Kcal(x_1,x_2)$. Such a field exists. Indeed,
for any finite set of $x_1,x_2,\ldots,x_T$ our requirements imply the
Gaussian distribution with the mean of 0 and the covariance matrix of
$K$. These distributions satisfy the consistency requirements and thus
the Kolmogorov extension (or existence) theorem (see, e.g.,
\cite{processes_lamperti}, Appendix~1 for a proof
sketch\footnote{Strictly speaking, we do not need to construct the
  field for the whole $X$ in order to prove the theorem; is suffices
  to consider a finite-dimensional Gaussian distribution of
  $(z_{x_1},z_{x_2},\ldots,z_{x_T})$.}) can be applied to construct a
field over $X$.

Let $\varepsilon_x$ be a Gaussian field of mutually independent and
independent of $z_x$ random values with the mean of 0 and variance
$\sigma^2$. The existence of such a field can be shown using the same
Kolmogorov theorem. Now let $y_x=z_x+\varepsilon_x$. Intuitively,
$\varepsilon_x$ can be thought of as random noise introduced by
measurements of the original field $z_x$. The field $z_x$ is not
observable directly and we can possibly obtain only the values of
$y_x$.

The learning process can be thought of as estimating the values of the
field $y_x$ given the values of the field at sample points. One can
show that the conditional distribution of $z_x$ given a sample
$S=((x_1,y_1),(x_2,y_2),\ldots,(x_T,y_T))$ is Gaussian with the mean
of $\gamma_x^\RR=Y'(K+\sigma^2I)^{-1}k(x)$ and the variance
$d_x=\Kcal(x,x)-k'(x)(K+\sigma^2I)^{-1}k(x)$.  The conditional
distribution of $y_x$ is Gaussian with the same mean and the variance
$\sigma^2+\Kcal(x,x)-k'(x)(K+\sigma^2I)^{-1}k(x)$ (see
\cite{Rasmussen2006}, Section 2.2, p.~17).

If we let $a=\sigma^2$, we see that $\gamma_t^\RR$ and $a+d_t$ are,
respectively, the mean and the variance of the conditional
distributions for $y_{x_t}$ given the sample $S_t$.

\begin{remark}
\label{rem_diff}
Note that in the statement of the theorem there is no assumption that
the signals $x_t$ are pairwise different. Some of them may
coincide. In the probabilistic picture all $x$s must be different though, or
the corresponding probabilities make no sense.
This obstacle may be overcome in the following way. Let us replace the
domain $X$ by $X'=X\times\N$, where $\N$ is the set of positive
integers $\{1,2,\ldots\}$, and replace $x_t$ by $x'_t=(x_t,t)\in
X'$. For $X'$ there is a Gaussian field with the covariance function
$\Kcal'((x_1,t_1),(x_2,t_2))=\Kcal(x_1,x_2)$. The argument concerning
the probabilistic meaning of ridge regression stays for $\Kcal'$ on
$X'$.
We can thus assume that all $x_t$ are different.
\end{remark}

The proof of the identity is based on the Gaussian field
interpretation. Let us calculate the density of the joint distribution
of the variables $(y_{x_1},y_{x_2},\ldots,y_{x_T})$ at the point
$(y_1,y_2,\ldots,y_T)$. We will do this in three different ways: by
decomposing the density into a chain of conditional densities,
marginalisation, and, finally, direct calculation. Each method will
give us a different expression corresponding to a term in the
identity. Since all the three terms express the same density, they
must be equal.

\subsection{Conditional Probabilities}

We have
\begin{multline*}
p_{y_{x_1},y_{x_2},\ldots,y_{x_T}}(y_1,y_2,\ldots,y_T)=\\
p_{y_{x_T}}(y_T\mid y_{x_1}=y_1,y_{x_2}=y_2,\ldots,y_{x_{T-1}}=y_{T-1})\cdot\\
p_{y_{x_1},y_{x_2},\ldots,y_{x_{T-1}}}(y_1,y_2,\ldots,y_{T-1})\enspace.
\end{multline*}
Expanding this further yields
\begin{multline*}
p_{y_{x_1},y_{x_2},\ldots,y_{x_T}}(y_1,y_2,\ldots,y_T)=\\
p_{y_{x_T}}(y_T\mid y_{x_1}=y_1,y_{x_2}=y_2,\ldots,y_{x_{T-1}}=y_{T-1})\cdot\\
p_{y_{x_{T-1}}}(y_{T-1}\mid
y_{x_1}=y_1,y_{x_2}=y_2,\ldots,y_{x_{T-2}}=y_{T-2})\cdots
p_{y_{x_1}}(y_1)\enspace.
\end{multline*}
As we have seen before, the distribution for $y_{x_t}$ given that
$y_{x_1}=y_1,y_{x_2}=y_2,\ldots,y_{x_{t-1}}=y_{t-1}$ is
Gaussian with the mean of $\gamma_t^\RR$ and the variance of $d_t+\sigma^2$. Thus
\begin{multline*}
p_{y_{x_t}}(y_t\mid
y_{x_1}=y_1,y_{x_2}=y_2,\ldots,y_{x_{t-1}}=y_{t-1})=\\
\frac{1}{\sqrt{2\pi}}\frac{1}{\sqrt{d_t+\sigma^2}}e^{-\frac
  12\frac{(y_t-\gamma_t^\RR)^2}{d_t+\sigma^2}}
\end{multline*}
and
\begin{multline*}
p_{y_{x_1},y_{x_2},\ldots,y_{x_T}}(y_1,y_2,\ldots,y_T)=\\
\frac{1}{(2\pi)^{T/2}\sqrt{(d_1+\sigma^2)(d_2+\sigma^2)\ldots
    (d_T+\sigma^2)}}e^{-\frac 12
  \sum_{t=1}^T\frac{(\gamma_t^\RR-y_t)^2}{d_t+\sigma^2}}\enspace.
\end{multline*}

\subsection{Dealing with a Singular Kernel Matrix}

The expression for the second case looks particularly simple for
non-singular $K$. Let us show that this is sufficient to prove the
identity.

All the terms in the identity are in fact continuous functions of
${T(T+1)/2}$ values of $\Kcal$ at the pairs of points $x_i,x_j$,
$i,j=1,2,\ldots,T$. Indeed, the values of $\gamma_t^\RR$ in the
left-hand side expression are ridge regression predictions given by
respective analytic formula. Note that the coefficients of the inverse
matrix are continuous functions of the original matrix.

The optimal function minimising the second expression is in fact
$f_\RR(x)=\sum_{t=1}^Tc_t\Kcal(x_t,x)$, where the coefficients $c_t$
are continuous functions of the values of $\Kcal$. The reproducing
property implies that
\begin{equation*}
\|f_\RR\|^2=\sum_{i,j=1}^Tc_ic_j\langle\Kcal(x_i,\cdot),
\Kcal(x_j,\cdot)\rangle_\Fcal
=\sum_{i,j=1}^Tc_ic_j\Kcal(x_i,x_j)\enspace.
\end{equation*}

We can thus conclude that all the expressions are continuous in the
values of $\Kcal$. Consider the kernel
$\Kcal_\alpha(x_1,x_2)=\Kcal(x_1,x_2)+\alpha\delta(x_1,x_2)$, where
$\delta$ is as in (\ref{eq_delta}) and $\alpha>0$.  Clearly, $\delta$
is a kernel and thus $\Kcal_\alpha$ is a kernel. If all $x_t$ are
different (recall Remark~\ref{rem_diff}), kernel matrix for
$\Kcal_\alpha$ equals $K+\alpha I$ and therefore it is non-singular.

However the values of $\Kcal_\alpha$ tend to the corresponding values
of $\Kcal$ as $\alpha\to 0$.

\subsection{Marginalisation}

The method of marginalisation consists of introducing extra variables
to obtain the joint density in some manageable form and then
integrating over the extra variables to get rid of them. The variables
we are going to consider are $z_{x_1},z_{x_2},\ldots,z_{x_T}$.

Given the values of $z_{x_1},z_{x_2},\ldots,z_{x_T}$, the density of
$y_{x_1},y_{x_2},\ldots,y_{x_T}$ is easy to calculate. Indeed, given
$z$s all $y$s are independent and have the means of corresponding $z$s
and variances of $\sigma^2$, i.e.,
\begin{multline*}
p_{y_{x_1},y_{x_2},\ldots,y_{x_T}}(y_1,y_2,\ldots,y_T\mid
z_{x_1}=z_1,z_{x_2}=z_2,\ldots,z_{x_{T-1}}=z_{T-1})=\\
\frac 1{\sqrt{2\pi}}\frac 1\sigma e^{-\frac
    12\frac{(y_1-z_1)^2}{\sigma^2}}
\frac 1{\sqrt{2\pi}}\frac 1\sigma e^{-\frac
    12\frac{(y_2-z_2)^2}{\sigma^2}}\cdots
\frac 1{\sqrt{2\pi}}\frac 1\sigma e^{-\frac
    12\frac{(y_T-z_T)^2}{\sigma^2}}=\\
\frac 1{(2\pi)^{T/2}\sigma^T}e^{-\frac
  {1}{2\sigma^2}\sum_{t=1}^T(y_t-z_t)^2}
\end{multline*}
Since $z_{x_1},z_{x_2},\ldots,z_{x_T}$ have a joint Gaussian
distribution with the mean of 0 and covariance matrix $K_T$, their
density is given by
\begin{equation*}
  p_{z_{x_1},z_{x_2},\ldots,z_{x_T}}(z_1,z_2,\ldots,z_T)=\frac
  1{(2\pi)^{T/2}\sqrt{\det K_{T}}}e^{-\frac 12 Z'K_{T}^{-1}Z}\enspace,
\end{equation*}
where $Z=(z_1,z_2,\ldots,z_T)'$, provided $K_{T}$ is non-singular.

Using
\begin{multline*}
p_{y_{x_1},y_{x_2},\ldots,y_{x_T},z_{x_1},z_{x_2},\ldots,z_{x_T}}
(y_1,y_2,\ldots,y_T,z_1,z_2,\ldots,z_T)=\\
p_{y_{x_1},y_{x_2},\ldots,y_{x_T}}(y_1,y_2,\ldots,y_T\mid
z_{x_1}=z_1,z_{x_2}=z_2,\ldots,z_{x_{T-1}}=z_{T-1})\cdot\\
p_{z_{x_1},z_{x_2},\ldots,z_{x_T}}(z_1,z_2,\ldots,z_T)
\end{multline*}
and
\begin{multline*}
p_{y_{x_1},y_{x_2},\ldots,y_{x_T}}(y_1,y_2,\ldots,y_T)=\\
\int_{\R^T}p_{y_{x_1},y_{x_2},\ldots,y_{x_T},z_{x_1},z_{x_2},\ldots,z_{x_T}}
(y_1,y_2,\ldots,y_T,z_1,z_2,\ldots,z_T)dZ
\end{multline*}
we get
\begin{multline*}
p_{y_{x_1},y_{x_2},\ldots,y_{x_T}}(y_1,y_2,\ldots,y_T)=\\
\frac 1{(2\pi)^{T/2}\sigma^T}\frac1{(2\pi)^{T/2}\sqrt{\det K_{T}}}
\int_{\R^T}e^{-\frac{1}{2\sigma^2}\sum_{t=1}^T(y_t-z_t)^2-\frac 12 Z'K_{T}^{-1}Z}dZ\enspace.
\end{multline*}
To evaluate the integral we need the following proposition (see
\cite{beckenbach/bellman:1961}, Theorem~3 of Chapter~2) .
\begin{proposition}
\label{prop_int}
Let $Q(\theta)$ be a quadratic form of $\theta\in\R^n$ with the
positive-definite quadratic part, i.e.,
$Q(\theta)=\theta'A\theta+\theta'b+c$, where the matrix $A$ is
symmetric positive-definite. Then
\begin{equation*}
\int_{\R^n}e^{-Q(\theta)}d\theta=e^{-Q(\theta_0)}\frac{\pi^{n/2}}
{\sqrt{\det    A}}\enspace,
\end{equation*}
where $\theta_0=\arg\min_{\R^n}Q$.
\end{proposition}
The quadratic part of the form in our integral has the matrix $\frac
12K_{T}^{-1}+\frac 1{2\sigma^2}I$ and therefore
\begin{multline*}
p_{y_{x_1},y_{x_2},\ldots,y_{x_T}}(y_1,y_2,\ldots,y_T)=\\
\frac 1{(2\pi)^{T}\sigma^T\sqrt{\det K_{T}}}\frac{\pi^{T/2}}{\sqrt{\det
  (\frac 12K_{T}^{-1}+\frac 1{2\sigma^2}I)}}\cdot \\
e^{-\min_Z\left(\frac{1}{2\sigma^2}\sum_{t=1}^T(y_t-z_t)^2+\frac 12
  Z'K_{T}^{-1}Z\right)}
\end{multline*}
We have
\begin{align*}
  \sqrt{\det K_{T}}\sqrt{\det\left(\frac 12K_{T}^{-1}+\frac 1{2\sigma^2}I\right)}&=
\sqrt{\det\left(\frac 12I+\frac 1{2\sigma^2}K_{T}\right)}\\ &=
\frac 1{2^{T/2}\sigma^T}\sqrt{\det(K_{T}+\sigma^2I)}\enspace.
\end{align*}
Let us deal with the minimum. We will link it to
\begin{equation*}
M=\min_{f\in\Fcal}\left(
\sum_{t=1}^T(f(x_t)-y_t)^2+\sigma^2\|f\|^2_\Fcal\right)\enspace.
\end{equation*}
The representer theorem (see Proposition~\ref{prop_representer})
implies that the minimum from the definition of $M$ is achieved on 
a function of the form $f(x)=\sum_{t=1}^Tc_t\Kcal(x_t,\cdot)$. For the
column vector $Z(x)=(f(x_1),f(x_2),\ldots,f(x_T))'$ we have
$Z(x)=K_{T}C$, where $C=(c_1,c_2,\ldots,c_T)'$. Since $K_{T}$ is
supposed to be non-singular, there is a one-to-one correspondence
between $C$ and $Z(x)$; we have $C=K^{-1}_{T}Z(x)$ and
$\|f\|^2_\Fcal=C'K_{T}C=Z'(x)K_{T}^{-1}Z(x)$. We can minimise by $Z$
instead of $C$ and therefore
\begin{equation*}
\min_Z\left(\frac{1}{2\sigma^2}\sum_{t=1}^T(y_t-z_t)^2+\frac 12
Z'K_{T}^{-1}Z\right)=\frac{1}{2\sigma^2}M\enspace.
\end{equation*}
For the density we get the expression
\begin{equation*}
  p_{y_{x_1},y_{x_2},\ldots,y_{x_T}}(y_1,y_2,\ldots,y_T)=
  \frac 1{(2\pi)^{T/2}\sqrt{\det(K_{T}+\sigma^2I)}}e^{-\frac
    1{2\sigma^2}M}\enspace.
\end{equation*}

\subsection{Direct Calculation}

One can easily calculate the covariances of $y$s:
\begin{align*}
\cov(y_{x_1},y_{x_2}) &=
E(z_{x_1}+\varepsilon_{x_1})(z_{x_2}+\varepsilon_{x_2})\\
&=Ez_{x_1}z_{x_2}+E\varepsilon_{x_1}\varepsilon_{x_2}\\
&=\Kcal(x_1,x_2)+\sigma^2\delta(x_1,x_2)\enspace.
\end{align*}
Therefore, one can write down the expression
\begin{multline*}
p_{y_{x_1},y_{x_2},\ldots,y_{x_T}}(y_1,y_2,\ldots,y_T)=\\
\frac 1{(2\pi)^{T/2}\sqrt{\det(K_{T}+\sigma^2I)}}e^{-\frac 12 Y'_{T}(K_{T}+\sigma^2I)^{-1}Y_{T}}\enspace.
\end{multline*}

\subsection{Equating the Terms}

It remains to take the logarithms of the densities calculated in
different ways. We need the following matrix lemma.
\begin{lemma}
\label{lem_proddt}
\begin{equation*}
(d_1+\sigma^2)(d_2+\sigma^2)\ldots (d_T+\sigma^2)=\det(K_{T}+\sigma^2I)
\end{equation*}
\end{lemma}

\begin{proof}
The lemma follows from Frobenius's identity (see, e.g.,
\cite{matan_henderson}):
\begin{equation*}
\det
\begin{pmatrix}
A & u \\
v'& d
\end{pmatrix}=
(d-v'A^{-1}u)\det A \enspace,
\end{equation*}
where $d$ is a scalar and the submatrix $A$ is non-singular.

We have
\begin{align*}
\det(K_{T}+\sigma^2I) &= (\Kcal(x_T,x_T)+\sigma^2-k'_{T-1}(x_T)(
K_{T-1}+\sigma^2I)^{-1}k_{T-1}(x_T))\cdot\\
&{}\;\;\;\; \det(K_{T-1}+\sigma^2I)
\\
&= (d_T+\sigma^2)\det(K_{T-1}+\sigma^2I)\\
&=\ldots\\
&=(d_T+\sigma^2)(d_{T-1}+\sigma^2)\ldots (d_2+\sigma^2)(d_1+\sigma^2)\enspace.
\end{align*}
\end{proof}

We get
\begin{equation*}
\sum_{t=1}^T\frac{(\gamma_t^\RR-y_t)^2}{d_t+\sigma^2}=\frac 1{2\sigma^2}M=
Y'_T(K_{T}+\sigma^2I)^{-1}Y_T\enspace.
\end{equation*}
The theorem follows.

\section{Bayesian Merging Algorithm and an Alternative Proof of the Identity}
\label{sect_AA}

In this section we reproduce an alternative way (after
\cite{abs-0910-4683}) of obtaining the identity.

An advantage of this approach is that we do not need to consider
random fields. The use of probability is minimal; all probabilities in
this approach are no more than weights or predictions. This provides
an additional intuition to the proof.

\subsection{Prediction with Expert Advice}

Consider the standard prediction with expert advice framework. Let
outcomes $y_1,y_2,\ldots$ from an {\em outcome set} $\Omega$ occur
successively in discrete time. A {\em learner} tries to predict each
outcome and outputs a prediction $\gamma_t$ from a {\em prediction
  set} $\Gamma$ each time before it sees the outcome $y_t$. There is
also a pool $\Theta$ of {\em experts}; experts try to predict the
outcomes from the same sequence and their predictions
$\gamma_t^\theta$ are made available to the learner. The quality of
predictions is assessed by means of a {\em loss function} $\lambda:
\Gamma\times\Omega\to [0,+\infty]$.

The framework can be summarised in the following protocol:

\begin{protocol}
{\tt
\begin{tabbing}
  \quad\= for \=\quad\kill\\
  for $t=1,2,\ldots$\\
  \> experts $\theta\in\Theta$ announce predictions $\gamma_t^\theta
  \in \Gamma$\\
  \> learner outputs $\gamma_t \in \Gamma$\\
  \> reality announces $y_t \in \Omega$\\
  \> each expert $\theta\in\Theta$ suffers loss
  $\lambda(\gamma_t^\theta,y_t)$\\
  \> learner suffers loss  $\lambda(\gamma_t,y_t)$\\
  endfor
\end{tabbing}
}
\end{protocol}

The goal of the learner in this framework is to suffer the cumulative
loss $\Loss_T=\sum_{t=1}^T\lambda(\gamma_t,y_t)$ not much larger than
the cumulative loss of each expert
$\Loss_T(\theta)=\sum_{t=1}^T\lambda(\gamma_t^\theta,y_t)$.

In this paper we consider the game with the outcome set $\Omega=\R$
and the prediction set $\Gamma$ of all continuous density functions
on $\R$, i.e., continuous functions $\xi:\R\to [0,+\infty)$ such
  that $\int_{-\infty}^{+\infty}\xi(y)dy=1$. The loss function is
  negative logarithmic likelihood, i.e.,
  $\lambda(\xi,y)=-\ln\xi(y)$.

\subsection{Bayesian Merging Algorithm}

Consider the following merging algorithm for the learner. The
algorithm takes an initial distribution $P_0$ on the pool of experts
$\Theta$ as a parameter and maintains weights $P_t$ for experts
$\theta$.

\begin{protocol}
\label{protocol_bayes}
{\tt
\begin{tabbing}
  \quad\= for \=\quad\kill\\
  let $P_0^*=P_0$\\
  for $t=1,2,\ldots$\\
  \> read experts' predictions $\xi_t^\theta \in \Gamma$, $\theta
  \in\Theta$\\
  \> predict $\xi_t = \int_\Theta \xi_t^\theta P_{t-1}^*(d\theta)$\\
  \> read $y_t$\\
  \> update the weights $P_t(d\theta) = \xi_t^\theta(y_t) P_{t-1}(d\theta)$\\
  \> normalise the weights $P_t^*(d\theta) = P_t(d\theta)/\int_\Theta
  P_t(d\theta)$\\ 
  endfor
\end{tabbing}
}
\end{protocol}

If we consider an expert $\theta$ as a probabilistic hypothesis, this
algorithm becomes the Bayesian strategy for merging hypotheses. The
weights $P_t^*$ relate to $P_{t-1}^*$ as posterior probabilities to
prior probabilities assigned to the hypotheses. We will refer to the
algorithm as the Bayesian Algorithm (BA).

The algorithm can also be considered as a special case of the
Aggregating Algorithm (\cite{vovk_aggr, vovk_advice}, see also
\cite{vovk_cols}) going back to \cite{raznoe_santis}. It is easy to
check that the Aggregating Algorithm for these outcome set,
prediction set, and the loss function and the learning rate $\eta=1$
reduces to Protocol~\ref{protocol_bayes}. However we will not be using
the results proved for the Aggregating Algorithm in this paper.

After $t$ steps the weights become
\begin{equation}\label{eq:weights}
  P_t(d\theta) = e^{-\Loss_t(\theta)} P_0(d\theta)\enspace.
\end{equation}
The following lemma is a special case of Lemma~1 in
Vovk~\cite{vovk_cols}.  It shows that the cumulative loss of the BA is
an average of the experts' cumulative losses in a generalised sense
(as in, e.g., Chapter 3 of \cite{matan_hardy_inequalities}).
\begin{lemma}\label{lem:lossAPA}
  For any prior $P_0$ and any $t=1,2,\ldots$, the cumulative loss of
  the BA can be expressed as
  \begin{equation*}
    \Loss_t
    =
    -\ln \int_{\Theta} e^{- \Loss_t(\theta)} P_0(d\theta).
  \end{equation*}
\end{lemma}
\begin{proof}
The proof is by induction on $t$. For $t=0$ the equality is obvious
and for $t>0$ we have
  \begin{multline*}
    \Loss_t = \Loss_{t-1} - \ln\xi_t(y_t) = \\ -\ln \int_{\Theta}
    e^{- \Loss_{t-1}(\theta)} P_0(d\theta) - \ln \int_{\Theta}
    \xi^{\theta}_t(y_t) \frac{e^{- \Loss_{t-1}(\theta)}}{\int_{\Theta}
      e^{- \Loss_{t-1}(\theta)} P_0(d\theta)} P_0(d\theta)\\ = 
 -\ln\int_{\Theta} e^{- (-\ln\xi_t^\theta(y_t)+\Loss_{t-1}(\theta))} P_0(d\theta) =
-\ln \int_{\Theta} e^{- \Loss_t(\theta)} P_0(d\theta)
  \end{multline*}
  (the second equality follows from the inductive assumption, the
  definition of $\xi_t$, and (\ref{eq:weights})).
\end{proof}

\subsection{Linear Ridge Regression as a Mixture}
\label{subsect_linearRR}

The above protocols can incorporate signals as in
Protocol~\ref{prot_online}. Indeed let the reality announce a signal
$x_t$ on each step $t$; the signal can be used by both the experts and
the learner.

Suppose that signals come from $\R^n$. Take a pool of {\em Gaussian
  experts} $\Theta=\R^n$. Fix some $\sigma>0$ and let expert $\theta$
output the density of Gaussian distribution
$\calN(\theta'x_t,\sigma^2)$, i.e.,
\begin{equation}
\label{eq_gaussian_experts}
\xi_t^\theta(y)=\frac 1{\sqrt{2\pi\sigma^2}}e^{-\frac{(\theta'x_t-y)^2}{2\sigma^2}}
\enspace, 
\end{equation}
on step $t$.

Let us assume the multivariate Gaussian distribution
$\calN(0,I)$ with the density
\begin{equation}
\label{eq_normal_prior}
p_0(\theta)=\frac 1{(2\pi)^{n/2}}e^{-\|\theta\|^2/2}
\end{equation}
as the initial distribution over the pool of experts. We will show
that the learner using the Bayesian merging algorithm with this
initial distribution will be outputting a Gaussian density with the
mean of the ridge regression prediction. Note that there is no
assumption on the mechanism generating outcomes $y_t$.

Let $Y_{t}$ be the vector of outcomes $y_1,y_2,\ldots,y_{t}$. Let
$X_{t}$ be the design matrix made up of column vectors
$x_1,x_2,\ldots,x_{t}$ and $A_t=X_tX'_t+\sigma^2I$, $t=1,2,\ldots$.

\begin{lemma}
\label{lem_RR_mix}
The learner using the Bayesian merging algorithm with the initial
distribution (\ref{eq_normal_prior}) on the pool of experts $\R^n$
predicting according to (\ref{eq_gaussian_experts}) will be outputting
on step $T=1,2,\ldots$ the density
\begin{equation*}
\xi_T(y)=\frac 1{\sqrt{2\pi\sigma_T^2}}e^{-\frac{(\gamma_T^\RR-y)^2}{2\sigma_T^2}}\enspace,
\end{equation*}
where
\begin{align*}
\gamma_T^\RR &= Y'_{T-1}X'_{T-1}A_{T-1}^{-1}x_T\\
\sigma^2_T&=\sigma^2x'_TA^{-1}_{T-1}x_T+\sigma^2\enspace.
\end{align*}
\end{lemma}

We have $\gamma_T^\RR=\left(\theta^{\RR}_{T}\right)' x_T$, where
$\theta_T^\RR =A_{T-1}^{-1}X_{T-1}Y_{T-1}$. At $\theta_T^\RR $ the
minimum $\min_{\theta\in\R^n}\left(\sum_{t=1}^{T-1}(\theta'x_t-y_t)^2+
\sigma^2\|\theta\|^2\right)$ is achieved. This can be checked directly
by differentiation or by reducing to Proposition~\ref{prop_KRR} (see
Subsection~\ref{subsect_kernelisation} for a discussion of linear
ridge regression as a special case of kernel ridge regression).
We will refer to the function $\left(\theta^{\RR}_{T}\right)' x$ as
the {\em linear ridge regression} with the parameter $\sigma^2$. We
are considering the on-line mode, but linear ridge regression can also
be applied in the batch mode just like the general kernel ridge
regression.

Let us prove the lemma.

\begin{proof}
To evaluate the integral
\begin{equation}
\xi_T(v) = \int_{\R^n}\xi_T^\theta(v)P_{T-1}^*(d\theta)
\end{equation}
we will use a probabilistic interpretation. 

Let $\theta$ be a random value distributed according to $P_{T-1}^*$,
i.e., having the density
$$
p_\theta(u)\sim e^{-\frac
  1{2\sigma^2}\sum_{t=1}^{T-1}(u'x_t-y_t)^2-\frac 12\|u\|^2}\enspace.
$$ Clearly, $\theta$ has a multivariate Gaussian distribution. The mean
of a Gaussian distribution coincides with its mode and thus the mean
of $\theta$ equals
$$\theta_T^\RR =\arg\min_{u\in\R^n}\left(\sum_{t=1}^{T-1}(u'x_t-y_t)^2+\sigma^2\|u\|^2\right)
=A_{T-1}^{-1}X_{T-1}Y_{T-1}\enspace.$$ The covariance matrix
$$ 
\Sigma=\left(\frac
1{\sigma^2}\sum_{t=1}^{T-1}x_tx'_t+I\right)^{-1}
=\sigma^2A_{T-1}^{-1}
$$ 
can be obtained by singling out the quadratic part of the quadratic
form in $u$.

Let $y$ be a random value given by $y=\theta'x_T+\varepsilon$, where
$\varepsilon$ is independent of $\theta$ and has a Gaussian
distribution with the mean of 0 and variance of $\sigma^2$. Clearly,
given that $\theta=u$, the distribution of $y$ is
$\calN(u'x_T,\sigma^2)$. The marginal density of $y$ is just
$\xi_T(v)$ we need to evaluate.

We will use the following statement from \cite{Bishop2006},
Section~2.3.3. Let $\eta$ have the (multivariate) Gaussian distribution
$\calN(\mu,\Lambda^{-1})$ and $\zeta$ have the (multivariate) Gaussian
distribution $\calN(A\eta+b,L^{-1})$, where $A$ is a fixed matrix and
$b$ is a fixed vector. Then the marginal distribution of $\zeta$ is
$\calN(A\mu+b,L^{-1}+A\Lambda^{-1}A')$.

We get that the mean of $y$ is $x'_T\theta_T^\RR $ and the variance is
$\sigma^2+x'_T\sigma^2A_{T-1}^{-1}x_T$.
\end{proof}

The lemma is essentially equivalent to the following statement from
Bayesian statistics. Let $y_t=x'_t\theta+\varepsilon_t$, where
$\varepsilon_t$ are independent Gaussian values with the means of 0
and variances $\sigma^2$, and $x_t$ are not stochastic.  Let the prior
distribution for $\theta$ be $\calN(0,I)$. Then the distribution for
$y_T$ given the observations
$x_1,y_1,x_2,y_2,\ldots,x_{T-1},y_{T-1},x_T$ is
$\calN(\gamma_T^\RR,\sigma_T)$; see, e.g., \cite{Bishop2006},
Section~3.3.2 or \cite{regression_hoerl_kennard}.

\subsection{The Identity in the Linear Case}

The following theorem is a special case of Theorem~\ref{th_main}

\begin{theorem}
\label{th_linear} Take $a>0$.
For a sample $(x_1,y_1),(x_2,y_2),\ldots,(x_T,y_T)$, where
$x_1,x_2,\ldots,x_T\in\R^n$ and $y_1,y_2,\ldots,y_T\in\R^n$, let
$\gamma_1^\RR,\gamma_2^\RR,\ldots,\gamma_T^\RR$ be the predictions
output by linear ridge regression with the parameter $a$ in the
on-line mode. Then
\begin{equation*}
\sum_{t=1}^T\frac{(\gamma^\RR_t-y_t)^2}{1+x'_tA_{t-1}^{-1}x_t}=\min_{\theta\in\R^n}
\left(\sum_{t=1}^T(\theta'x_t-y_t)^2+a\|\theta\|^2\right)=
aY'_T(X'_TX_T+aI)^{-1}Y_T
\enspace,
\end{equation*}
where $A_t=\sum_{i=1}^{t}x'_ix_i+aI=X'_tX_t$, $X_t$ is the design
matrix consisting of column vectors $x_1,x_2,\ldots,x_t$, and
$Y_t=(y_1,y_2,\ldots,y_t)'$.
\end{theorem}

\begin{proof}
We start by showing that the first two terms are equal and then
proceed to the last term.

Consider the pool of Gaussian experts with the variance $\sigma^2=a$
and the learner following the Bayesian merging algorithm with the
initial distribution $\calN(0,I)$ on the experts. 

It follows from Lemma~\ref{lem_RR_mix}, that the total loss of the
learner over $T$ steps is given by
\begin{align}
\Loss_T &= -\sum_{t=1}^T\ln\frac
1{\sqrt{2\pi\sigma_t^2}}e^{-\frac{(\gamma_t-y_t)^2}{2\sigma_t^2}}\\
\label{eq_loss_def}
&=\sum_{t=1}^T\frac{(\gamma_t-y_t)^2}{2\sigma_t^2}+
\ln\prod_{t=1}^T\sigma_t+\frac T2\ln(2\pi)\enspace,
\end{align}
where $\sigma^2_t=\sigma^2(1+x'_tA^{-1}_{t-1}x_t)$.

Lemma~\ref{lem:lossAPA} implies that
\begin{equation*}
\Loss_T = -\ln\left(\frac 1{(2\pi\sigma^2)^{T/2}(2\pi)^{n/2}}
\int_{\R^n}
e^{-\frac{1}{2\sigma^2}\sum_{t=1}^T(\theta'x_t-y_t)^2-\frac
  12\|\theta\|^2}d\theta\right)\enspace.
\end{equation*}
It follows from Proposition~\ref{prop_int} that the integral evaluates
to
\begin{multline*}
e^{-\min_{\theta\in\R_n}\left(\frac{1}{2\sigma^2}\sum_{t=1}^T(\theta'x_t-y_t)^2+
\frac
12\|\theta\|^2\right)}\frac{\pi^{n/2}}{\sqrt{\det(A_T/(2\sigma^2))}}=\\
e^{-\frac 1{2\sigma^2}\min_{\theta\in\R_n}\left(\sum_{t=1}^T(\theta'x_t-y_t)^2+
\sigma^2\|\theta\|^2\right)}\frac{(2\pi)^{n/2}}{\sqrt{\det(A_T/\sigma^2)}}
\end{multline*}
and thus
\begin{equation}
\label{eq_loss_min}
\begin{split}
\Loss_T =\frac
1{2\sigma^2}\min_{\theta\in\R_n}\left(\sum_{t=1}^T(\theta'x_t-y_t)^2+
\sigma^2\|\theta\|^2\right)+\\
\frac T2\ln(2\pi)+T\ln\sigma+\frac
12\ln\det\frac{A_T}{\sigma^2}\enspace.
\end{split}
\end{equation}

Let us equate the expressions for the loss provided by
(\ref{eq_loss_def}) and (\ref{eq_loss_min}). To prove the identity we
need to show that
$$
\frac 12\ln\prod_{t=1}^T\sigma^2_t=\frac 12 T\ln\sigma^2+\frac
12\ln\det\frac{A_T}{\sigma^2}\enspace.
$$
This equality follows from the lemma.
\begin{lemma}
For any $a>0$ and positive integer $T$ we have
$$
\det\frac{A_T}a=\prod_{t=1}^T(1+x'_tA_{t-1}^{-1}x_t)\enspace,
$$
where $A_t=\sum_{i=1}^{t}x'_ix_i+aI$.
\end{lemma}

\begin{proof}
We will use the matrix determinant lemma
$$\det(A+uv')=(1+v'A^{-1}u)\det A\enspace,$$ which holds for any non-singular
square matrix $A$ and vectors $u$ and $v$ (see, e.g.,
\cite{Harville1997}, Theorem~18.1.1). We get
\begin{align*}
\det\frac{A_T}a&= \frac 1{a^n}\det(A_{T-1}+x_Tx'_T)\\
&= \frac 1{a^n}\det(A_{T-1})(1+x'_TA_{T-1}^{-1}x_T)\\
&= \ldots\\
&= \frac 1{a^n}\det(aI)\prod_{t=1}^T(1+x'_tA_{t-1}x_t)\\
&= \prod_{t=1}^T(1+x'_tA_{t-1}^{-1}x_t)\enspace.
\end{align*}
\end{proof}

\begin{remark}
The lemma is in fact a special case of Lemma~\ref{lem_proddt} with the
linear kernel $\Kcal(u_1,u_2)=u'_1u_2$ and $a=\sigma^2$.  As shown in
Subsection~\ref{subsect_kernelisation},
$d_t/a=x'_tA^{-1}_{t-1}x_t$. The Sylvester identity implies that
$\det(aI+K_T)=\det(aI+X'_TX_T)=\det(aI+X_TX'_T)=\det A_T$.
\end{remark}

We have shown that the left-hand side and the middle terms in the
identity are equal. Let us proceed to the equality between the middle
and the right-hand side terms.

The minimum in the middle term is achieved on
$\theta_{T+1}^\RR=A_{T}^{-1}X_TY_T=(X_TX'_T+aI)^{-1}X_TY_T$ as shown
in Subsection~\ref{subsect_linearRR}. Using
Lemma~\ref{lemma_matrix_inv} we can also write
$\theta_{T+1}^\RR=X_T(X'_TX_T+aI)^{-1}Y_T$.  The proof is by direct
substitution of these expressions for $\theta_{T+1}^\RR$. We
have
\begin{multline*}
M=\min_{\theta\in\R^n}\left(\sum_{t=1}^T(\theta'x_t-y_t)^2+a\|\theta\|^2\right)=
\sum_{t=1}^T\left(\left(\theta_{T+1}^\RR\right)'x_t-y_t\right)^2+
a\|\theta_{T+1}^\RR\|^2=\\
\left(\theta_{T+1}^\RR\right)'(X_TX'_T+aI)\theta_{T+1}^\RR
-2\left(\theta_{T+1}^\RR\right)'X_TY_T+Y'_TY_T\enspace.
\end{multline*}
Substituting the first expression for the second appearance of
$\theta_{T+1}^\RR$ and cancelling out $X_TX'_T+aI$ we get
\begin{equation*}
M=(-\theta_{T+1}^\RR X_T+Y'_T)Y_T\enspace.
\end{equation*}
Substituting the second expression for $\theta_{T+1}^\RR$ yields
\begin{equation*}
M=Y'_T(-(X'_TX_T+aI)^{-1}X'_TX_T+I)Y_T
\end{equation*}
It remains to carry $(X'_TX_T+aI)^{-1}$ out of the brackets and cancel
out the remaining terms.
\end{proof}

\subsection{Kernelisation}
\label{subsect_kernelisation}

Let us derive Theorem~\ref{th_main} from Theorem~\ref{th_linear}.

First, let us show that Theorem~\ref{th_linear} is really a special
case of Theorem~\ref{th_main} for the linear kernel
$\Kcal(x_1,x_2)=x'_1x_2$. We will consider the identity term by
term. By Lemma~\ref{lemma_matrix_inv} the prediction output by linear
ridge regression on step $t$ equals
\begin{align*}
\left(\theta_t^\RR\right)'x_t &=
Y'_{t-1}X'_{t-1}(X_{t-1}X'_{t-1}+aI)^{-1}x_t\\
&= Y'_{t-1}(X'_{t-1}X_{t-1}+aI)^{-1}X_{t-1}x_t\\
&= Y'_{t-1}(K_{t-1}+aI)^{-1}k(x_t)\enspace.
\end{align*}
For the linear kernel the expression $d_t/a$ in the denominator of the
identity can be rewritten as follows:
\begin{align*}
\frac{d_t}{a}
&=\frac 1a\left[\Kcal(x_t,x_t)-k'_{t-1}(x_t)(K_{t-1}+aI)^{-1}k_{t-1}(x_t)\right]\\
&=\frac 1a\left[x'_tx_t-(x'_tX_{t-1})(X_{t-1}'X_{t-1}+aI)^{-1}(X'_{t-1}x_t)\right]\enspace.
\end{align*}
We can apply Lemma~\ref{lemma_matrix_inv} and further obtain
\begin{align}
\notag
\frac{d_t}{a} &=\frac 1a\left[x'_tx_t-x'_t(X_{t-1}X'_{t-1}+aI)^{-1}X_{t-1}X'_{t-1}x_t\right]\\
\notag
&=\frac
1a\left[x'_t(I-(X_{t-1}X'_{t-1}+aI)^{-1}X_{t-1}X'_{t-1})x_t\right]\\
\label{dt_linear}
&=x'_t(X_{t-1}X'_{t-1}+aI)^{-1}x_t\\
\notag
&=x'_tA^{-1}_{t-1}x_t\enspace.
\end{align}

Let us proceed to the middle term in the identity. The set of
functions $f_\theta(x)=\theta'x$ on $\R^n$ with the scalar product
$\langle f_{\theta_1},f_{\theta_2}\rangle=\theta_1'\theta_2$ is a
Hilbert space. It contains all functions $\Kcal(u,\cdot)=f_u$ and the
reproducing property for $\Kcal$ holds: $\langle
f_\theta,\Kcal(x,\cdot)\rangle= \langle
f_\theta,f_x\rangle=\theta'x=f_\theta(x)$. The
minimum in the middle term of Theorem~\ref{th_linear} is thus the
same as in the middle term of Theorem~\ref{th_main}.

For the right-hand side term the equality is obvious.

Now take an arbitrary kernel $\Kcal$ on a domain $X$ and let $\Fcal$
be the corresponding RKHS. We will apply a standard kernel trick.
Consider a sample $(x_1,y_1),(x_2,y_2),\ldots,(x_T,y_T)$, where
$x_t\in X$ and $y_t\in\R$, $t=1,2,\ldots,T$. It follows from the
representer theorem (see Proposition~\ref{prop_representer}) that the
minimum in the middle term is achieved on a linear combination of the
form $f(\cdot)=\sum_{t=1}^Tc_t\Kcal(x_t,\cdot)$, where
$c_1,c_2,\ldots,c_T\in\R$. These linear combinations form a
finite-dimensional subspace in the RKHS $\Fcal$. Let
$e_1,e_2,\ldots,e_m$, $m\le T$, be its orthonormal base and let
$\calC$ map each linear combination $f$ into the (column) vector of
its coordinates in $e_1,e_2,\ldots,e_m$. Since the base is
orthonormal, the scalar product does not change and $\langle
f_1,f_2\rangle_\Fcal=(\calC(f_1))'\calC(f_2)$. The
reproducing property implies that
$$f(x_t)=\langle f,\Kcal(x_t,\cdot)\rangle_\Fcal=
(\calC(f))'\calC(\Kcal(x_t,\cdot))$$
for $t=1,2,\ldots,T$.  We also have
$$\Kcal(x_i,x_j)=\langle\Kcal(x_i,\cdot),\Kcal(x_j,\cdot)\rangle_\Fcal=
(\calC(\Kcal(x_i,\cdot)))'\calC(\Kcal(x_j,\cdot)\enspace,$$
$i,j=1,2,\ldots,T$. Note that $\calC$ is a surjection: each
$\theta\in\R^m$ is an image of some linear combination $f$.

Consider the sample $(\tilde x_1,y_1),(\tilde x_2,y_2),\ldots,(\tilde
x_T,y_T)$, where $\tilde x_t=\calC(\Kcal(x_t,\cdot))\in\R^m$,
$t=1,2,\ldots,T$.  Clearly, linear ridge regression in the on-line
mode outputs the same predictions on this sample as the kernel ridge
regression on the original sample and $\langle\tilde x_i,\tilde
x_j\rangle=\Kcal(x_i,x_j)$. The minimum from Theorem~\ref{th_main} on
the original sample clearly coincides with the minimum from
Theorem~\ref{th_linear} on the new sample.

 Theorem~\ref{th_main} follows.

\appendixsection{Optimality of Kernel Ridge Regression}
\label{appendix_optimality}

In this appendix we derive the optimality property for the ridge
regression function function $f_\RR$.

\begin{proposition}
\label{prop_KRR}
Let $\Kcal: X\times X\to\R$ be a kernel on a domain $X$ and $\Fcal$ be
the corresponding RKHS. For every non-negative integer $T$, every
$x_1,x_2,\ldots,x_T\in X$ and $y_1,y_2,\ldots,y_T\in\R$, and every
$a>0$ the minimum
\begin{equation}
\label{eq_min}
\min_{f\in\Fcal}\left(
\sum_{t=1}^T(f(x_t)-y_t)^2+a\|f\|^2_\Fcal\right)
\end{equation}
is achieved on the unique function $f_\RR(x)=Y'(aI+K)^{-1}k(x)$ for
$T>0$, where $Y$, $K$, and $k(x)$ are as in
Subsection~\ref{subsect_RR}, and $f_\RR(x)=0$ identically for $T=0$
\end{proposition}

\begin{proof}

If $T=0$, i.e., the initial sample is empty, the sum in~(\ref{eq_min})
contains no terms and the minimum is achieved on the unique function
$f$ with the norm $\|f\|_\Fcal=0$. This function is identically equal
to zero and it coincides with $f_\RR$ for this case by definition. For
the rest of the proof assume $T>0$.

The representer theorem (see Proposition~\ref{prop_representer})
implies that every minimum in~(\ref{eq_min}) is achieved on a linear
combination of the form $f(\cdot)=\sum_{t=1}^Tc_t\Kcal(x_t,\cdot)$.

The minimum in (\ref{eq_min}) thus can be taken over a finite-dimensional
space. As $\|f\|_\Fcal\to\infty$, the expression tends to $+\infty$,
and thus the minimum can be taken over a bounded set of functions. The
value $f(x)=\langle f, \Kcal(x,\cdot)\rangle_\Fcal$ is continuous in
$f$ for every $x\in X$. Therefore we are minimising a continuous
function over a bounded set in a finite-dimensional space. The minimum
must be achieved on some $f$.

Let $C=(c_1,c_2,\ldots,c_T)'$ be the vector of coefficients of some
optimal function $f(x)=\sum_{t=1}^Tc_t\Kcal(x_t,x)=C'k(x)$. It is easy to see
that the vector $(f(x_1), f(x_2), \ldots, f(x_T))'$ of values of $f$
equals $KC$ and
\begin{align*}
\|f\|^2_\Fcal&=
\sum_{i,j=1}^Tc_ic_j\langle\Kcal(x_i,\cdot),\Kcal(x_j,\cdot)\rangle_\Fcal\\
&= C'KC\enspace.
\end{align*}
Thus 
\begin{align*}
\sum_{t=1}^T(f(x)-y_t)^2+a\|f\|^2_\Fcal&=\|KC-Y\|^2+aC'KC\\
&=C'K^2C-2Y'KC+\|Y\|^2+aC'KC\enspace.
\end{align*}
Since $f$ is optimal, the derivative over $C$ must vanish.  By
differentiation we obtain
\begin{align*}
2K^2C-2KY+2aKC&=0\\
\intertext{and}
K(K+aI)C &=KY\enspace.
\end{align*}
Hence 
\begin{equation*}
(K+aI)C=Y+v
\end{equation*}
and
\begin{equation*}
C =(K+aI)^{-1}Y+(K+aI)^{-1}v\enspace,
\end{equation*}
where $v$ belongs to the null space of $K$, i.e., $Kv=0$.

Let us show that $K(K+aI)^{-1}v=0$. We need a simple matrix identity;
as it occurs in this paper quite often, we formulate it explicitly.
\begin{lemma}
\label{lemma_matrix_inv}
For any (not necessarily square) matrices $A$ and $B$ and any constant
$a$ the identity 
$$
A(BA+aI)^{-1}=(AB+aI)^{-1}A
$$
holds provided the inversions can be performed. If $B=A'$ and $a>0$,
the matrices $AB+aI$ and $BA+aI$ are both positive-definite and
therefore non-singular.
\end{lemma}

\begin{proof}
We have $ABA+aA=A(BA+aI)=(AB+aI)A$. If $AB+aI$ and $BA+aI$ are
invertible, we can multiply the equality by the inverses.
\end{proof}

We get $K(K+aI)^{-1}v=(K+aI)^{-1}Kv=0$. Therefore $C$ has the form
$C=(K+aI)^{-1}Y+u$, where $Ku=0$.

Consider the function $f_u(x)=u'k(x)$. It is a linear combination of
$\Kcal(x_i,\cdot)$. On the other hand, it vanishes at every $x_t$,
$t=1,2,\ldots,T$ because $Ku=0$. We have
\begin{equation*}
0 = f_u(x_t)= \langle f,\Kcal(x_t,\cdot)\rangle_\Fcal
\end{equation*}
and thus $f_u$ is orthogonal to the space of linear combinations. This
is only possible if $f_u=0$.

Thus the minimum can only be achieved on a unique function that can be
represented as $f_\RR(x)=Y'(K+aI)^{-1}k(x)$. Since it must be
achieved somewhere, it is achieved on $f_\RR$.
\end{proof}

\appendixsection{Representer Theorem}
\label{appendix_reproducing}

In this appendix we formulate and prove a version of the reproducing
property for RKHSs. See \cite{kern_generalized_representer} for more
details including a history of the theorem.

\begin{proposition}
\label{prop_representer}
Let $\Kcal$ be a kernel on a domain $X$, $\Fcal$ be the corresponding
RKHS and $(x_1,y_1),(x_2,y_2),\ldots,(x_T,y_T)$ be a sample such that
$x_t\in X$ and $y_t\in\R$, $t=1,2,\ldots,T$. Then for every
$f\in\Fcal$ there is a linear combination $\tilde
f(\cdot)=\sum_{t=1}^Tc_t\Kcal(x_t,\cdot)\in\Fcal$ such that
$$ 
\sum_{t=1}^T(\tilde f(x_t)-y_i)^2\le\sum_{t=1}^T(f(x_t)-y_i)^2
$$ and $\|\tilde f\|_\Fcal\le\|f\|_\Fcal$. 
If $f$ is not itself a linear combination of this type, there is a
linear combination $\tilde f$ with this property such that $\|\tilde
f\|_\Fcal<\|f\|_\Fcal$.
\end{proposition}

\begin{proof}
The linear combinations of $\Kcal(x_t,\cdot)$ form a
finite-dimensional (and therefore closed) subspace in the Hilbert
space $\Fcal$. Every $f\in\Fcal$ can be represented as $f=h+g$, where
$h$ is a linear combination and $g$ is orthogonal to the subspace of
linear combinations. For every $t=1,2,\ldots,T$ we have
$g(x_t)=\langle g,\Kcal(x_t,\cdot)\rangle_\Fcal=0$ and the values of
$f$ and $h$ on $x_1,x_2,\ldots,x_T$ coincide. On the other hand, the
Pythagoras theorem implies that
$\|f\|_\Fcal^2=\|h\|_\Fcal^2+\|g\|_\Fcal^2\ge \|h\|_\Fcal^2$; if $g\ne
0$, the inequality is strict.
\end{proof}

\appendixsection{An Upper Bound on a Determinant}
\label{appendix_cb_bound}

In this appendix we reproduce an upper bound from
\cite{cesa-bianchi_perceptron}.

\begin{proposition}
Let the columns of a $n\times T$ matrix $X$ be vectors
$x_1,x_2,\ldots,x_T\in\R^n$ and $a>0$. If $\|x_t\|\le B$,
$t=1,2,\ldots,T$, then
$$
\det\left(I+\frac 1a XX'\right)=\det\left(I+\frac 1a X'X\right)\le
\left(1+\frac{TB^2}{an}\right)^n
\enspace.
$$
\end{proposition}

\begin{proof}
Let $\lambda_1,\lambda_2,\ldots,\lambda_n\ge 0$ be the eigenvalues
(counting multiplicities) of the symmetric positive-definite matrix
$XX'$.  The eigenvalues of $I+\frac 1a XX'$ are then
$1+\lambda_1/a,1+\lambda_2/a,\ldots,1+\lambda_n/a$ and $\det(I+\frac
1a XX')=\prod_{i=1}^n(1+\frac{\lambda_i}a)$.

The sum of eigenvalues $\lambda_1+\lambda_2+\ldots+\lambda_n$ equals
the trace $\tr(XX')$ and $\tr(XX')=\tr(X'X)$. Indeed, the matrices
$AB$ and $BA$ (provided they exist) have the same non-zero eigenvalues
counting multiplicities while zero eigenvalues do not contribute to
the trace. Alternatively one can verify the equality $\tr(AB)=\tr(BA)$
by a straightforward calculation, see, e.g., \cite{matan_axler},
Proposition~10.9 (p.~219). The matrix $X'X$ is the Gram matrix of
vectors $x_1,x_2,\ldots,x_T$ and the elements on its diagonal are the
squared quadratic norms of the vectors not exceeding $B^2$. We get
$\tr(XX')=\tr(X'X)\le TB^2$.

The problem has reduced to obtaining an upper bound on the product of
some positive numbers with a known sum. The inequality of arithmetic
and geometric means implies that
\begin{equation*}
\prod_{i=1}^n\left(1+\frac 1a\lambda_i\right)\le
\left(\frac 1n\sum_{i=1}^n\left(1+\frac 1a\lambda_i\right)\right)^n
= \left(1+\frac 1{an} \sum_{i=1}^n\lambda_i\right)^n\enspace.
\end{equation*}
Combining this with the bound on the trace obtained earlier proves the
lemma.
\end{proof}

\appendixsection{A lemma about partitioned matrices}
\label{appendix_partition}

In this appendix we formulate and prove a matrix lemma for the proof
of Lemma~\ref{lemma_dt_to_0}.

\begin{lemma}
\label{lemma_partition}
If a symmetric positive-definite matrix $M$ is partitioned as
\begin{equation*}
M=
\begin{pmatrix}
A & B\\
B'& D 
\end{pmatrix}\enspace,
\end{equation*}
where the $A$ and $D$ are square matrices, and a column vector $x$ is
partitioned as
\begin{equation*}
x=
\begin{pmatrix}
u \\
v
\end{pmatrix}\enspace,
\end{equation*}
where $u$ is of the same height as $A$, then $x'M^{-1}x \ge
u'A^{-1}u\ge 0$.
\end{lemma}

\begin{proof}
We shall rely on the following formula for inverting a partitioned
matrix: if
\begin{equation*}
M=
\begin{pmatrix}
P & Q\\
R & S\\
\end{pmatrix}
\end{equation*}
then the inverse can be written as
\begin{equation*}
M^{-1}=
\begin{pmatrix}
\tilde P & \tilde Q\\
\tilde R & \tilde S\\
\end{pmatrix}\enspace,
\end{equation*}
where
\begin{align*}
\tilde P & = P^{-1}+P^{-1}Q(S-RP^{-1}Q)^{-1}RP^{-1}\enspace,\\
\tilde Q & = -P^{-1}Q(S-RP^{-1}Q)^{-1}\enspace,\\
\tilde R & = -(S-RP^{-1}Q)^{-1}RP^{-1}\enspace,\\
\tilde S & = (S-RP^{-1}Q)^{-1}\enspace,
\end{align*}
provided all the inverses exist (see \cite{raznoe_numrecpp},
Section~2.7.4, equation~(2.7.25)). Applying these formulae to
our partitioning of $M$ we get
\begin{equation*}
M^{-1}=
\begin{pmatrix}
A^{-1}+A^{-1}BE^{-1}B'A^{-1} & -A^{-1}BE^{-1}\\
-E^{-1}B'A^{-1} & E^{-1} \\
\end{pmatrix}\enspace,
\end{equation*}
where $E = D-B'A^{-1}B$. 

The matrix $A$ is symmetric positive-definite as a minor of a
symmetric positive-definite matrix; therefore it is
non-singular. Non-singularity of $E$ follows from the identity
\begin{equation*}
\det M = \det P \det(S-RP^{-1}Q)\enspace,
\end{equation*}
where $M$ and its blocks are as above (see \cite{raznoe_numrecpp},
Section~2.7.4, equation~(2.7.26) and \cite{matan_hornjohnson},
Section~0.8.5; the matrix $S-RP^{-1}Q$ is known as the Schur
complement of $P$). Applying this identity to our matrices yields
\begin{equation*}
\det M=\det A\det E
\end{equation*}
and since both $M$ and $A$ are non-singular, $E$ is also
non-singular. This justifies the use of the formula for the inverse of
a partitioned matrix in this case.

Note also that $E^{-1}$ is symmetric and positive-definite as a minor
of a symmetric positive-definite matrix
$M^{-1}$.

We can now write
\begin{equation*}
x'Mx= 
u'A^{-1}u+\\u'A^{-1}BE^{-1}B'A^{-1}u-2u'A^{-1}BE^{-1}v+v'E^{-1}v
\end{equation*}
(since $u'A^{-1}BE^{-1}v$ is a number, it equals its transpose). The
first term in the sum is just what we need for the statement of the
lemma.  Let us show that the sum of the remaining three terms is
non-negative.  Let $w= B'A^{-1}u$. We have
\begin{multline*}
u'A^{-1}BE^{-1}B'A^{-1}u-2u'A^{-1}BE^{-1}v+v'E^{-1}v=\\
w'E^{-1}w-2w'E^{-1}v+v'E^{-1}v=
\begin{pmatrix}
w' & v'\\
\end{pmatrix}
\begin{pmatrix}
E^{-1} & -E^{-1}\\
-E^{-1} & E^{-1}  
\end{pmatrix}
\begin{pmatrix}
w\\
v
\end{pmatrix}\enspace.
\end{multline*}
To complete the proof, we need the following simple lemma.
\begin{lemma}
If a matrix $H$ is symmetric positive-semidefinite, then the matrix
\begin{equation*}
\begin{pmatrix}
H  & -H\\
-H & H   
\end{pmatrix}
\end{equation*}
is also symmetric positive-semidefinite.
\end{lemma}

\begin{proof}
We will rely on the following criterion. A symmetric matrix $H$ is
positive-semidefinite if and only if it is has a symmetric square
root $L$ such that $H=L^2$ (the if part is trivial and the only if
part can be proven by considering the orthonormal base where $H$
diagonalises). We have
\begin{equation*}
\begin{pmatrix}
\frac{L}{\sqrt{2}}  & -\frac{L}{\sqrt{2}}\\
-\frac{L}{\sqrt{2}} & \frac{L}{\sqrt{2}}    
\end{pmatrix}
\begin{pmatrix}
\frac{L}{\sqrt{2}}  & -\frac{L}{\sqrt{2}}\\
-\frac{L}{\sqrt{2}} & \frac{L}{\sqrt{2}}    
\end{pmatrix}=
\begin{pmatrix}
L^2  & -L^2\\
-L^2 & L^2
\end{pmatrix}\enspace.  
\end{equation*}
\end{proof}

Thus 
\begin{equation*}
\begin{pmatrix}
w' & v'\\
\end{pmatrix}
\begin{pmatrix}
E^{-1} & -E^{-1}\\
-E^{-1} & E^{-1}  
\end{pmatrix}
\begin{pmatrix}
w\\
v
\end{pmatrix}\ge 0\enspace.
\end{equation*}
\end{proof}

\subsection*{Acknowledgements}

The authors have been supported through the EPSRC grant EP/F002998
`Practical competitive prediction'. The first author has also been
supported by an ASPIDA grant from the Cyprus Research Promotion
Foundation.

The authors are grateful to Vladimir Vovk and Alexey Chernov for
useful discussions and to anonymous COLT and ALT reviewers for
detailed comments.

\bibliographystyle{alpha}
\bibliography{C:/Bib/kolm,%
C:/Bib/volodya,%
C:/Bib/matan,%
C:/Bib/majority,%
C:/Bib/me,%
C:/Bib/raznoe,%
C:/Bib/rissanen,%
C:/Bib/vyugin,%
C:/Bib/calif,%
C:/Bib/steve,%
C:/Bib/processes,%
C:/Bib/kalman,%
C:/Bib/volatility,%
C:/Bib/hutter,%
C:/Bib/regression,%
C:/Bib/statistika,%
C:/Bib/kernels}

\end{document}